\documentclass[]{fairmeta}

\usepackage{amsmath}
\usepackage{amssymb}
\usepackage{amsfonts}
\usepackage{nicefrac}
\usepackage{algorithm}
\usepackage{algorithmic}
\usepackage{enumitem}
\usepackage{tikz}
\usetikzlibrary{positioning, calc, arrows.meta, fit}
\usepackage{amsthm}
\usepackage{pifont}
\usepackage{wrapfig}

\newtheorem{theorem}{Theorem}
\newtheorem{proposition}[theorem]{Proposition}
\newtheorem{lemma}[theorem]{Lemma}
\newtheorem{corollary}[theorem]{Corollary}

\newtheorem{remark}[theorem]{Remark}

\def\1{\bm{1}}

\DeclareMathAlphabet{\mathsfit}{\encodingdefault}{\sfdefault}{m}{sl}
\SetMathAlphabet{\mathsfit}{bold}{\encodingdefault}{\sfdefault}{bx}{n}

\title{LoopFM: Learning frOm HistOrical RePresentations of Foundation Model for Recommendation}

\author[*]{Hua Zheng}
\author[*]{Shali Jiang}
\author[*]{Boyang Liu}
\author[]{Laming Chen}
\author[]{Kenny Lov}
\author[]{Chuanqi Xu}
\author[]{Lisang Ding}
\author[]{Qinghai Zhou}
\author[]{Can Cui}
\author[]{Xiaolong Liu}
\author[]{Xiaoyi Liu}
\author[]{Yasmine Badr}
\author[]{Xin Xu}
\author[]{Mingfu Liang}
\author[]{Jiyan Yang}
\author[]{Ellie Dingqiao Wen}
\author[]{Gerard Jonathan Mugisha Akkerhuis}
\author[]{Jason Rudy}
\author[]{Xi Liu}
\author[]{Chenxiao Guan}
\author[]{Rong Jin}
\author[]{Ruichao Qiu}
\author[]{Xian Chen}
\author[]{Zhehui Zhou}
\author[]{Ping Chen}
\author[]{Rui Yang}
\author[]{Haicheng Chen}
\author[]{Meet Raval}
\author[]{Song Zhou}
\author[]{Dharak Kharod}
\author[]{Shuyu Xu}
\author[]{Xingyuan Wang}
\author[]{Liang Tao}
\author[]{Qiang Jin}
\author[]{Qiao Yang}
\author[]{Wankun Zhu}
\author[]{Qin Huang}
\author[]{Yuzhen Huang}
\author[]{Darren Liu}
\author[]{Parish Aggarwal}
\author[]{Hui Zhou}
\author[]{Erzhuo Wang}
\author[]{Shuo Chang}
\author[]{Xiaorui Gan}
\author[]{Wenlin Chen}
\author[]{Santanu Kolay}
\author[\dagger]{Huayu Li}

\contribution[*]{Equal contribution}
\contribution[\dagger]{Correspondence, AI at Meta}
\abstract{Knowledge distillation (KD) transfers a single scalar prediction from a large foundation model (FM) to compact vertical models (VMs), suffering from diminishing transfer ratio---the fraction of FM improvement captured by the VM---as a single scalar cannot convey the rich intermediate knowledge that larger FMs learn. To address this bottleneck, we propose \textbf{LoopFM} (\textbf{L}earning fr\textbf{O}m Hist\textbf{O}rical Re\textbf{P}resentations of \textbf{FM}), a framework that opens a high-bandwidth transfer channel by structuring FM intermediate embeddings as \emph{input features} (e.g., user history sequence) for downstream VMs, without requiring real-time FM inference at serving and architectural coupling between FM and VM.
We provide a theoretical framework for LoopFM with a gain decomposition and transfer-ratio analysis.
On three public benchmarks, LoopFM demonstrates strong AUC improvements (e.g., 6\%+ on TaobaoAd) and complementary knowledge transfer capability with KD. On industrial-scale systems (billions of examples, trillion-parameter FMs), LoopFM approximately \textbf{doubles} the knowledge transfer ratio on top of KD, delivering a \textbf{+0.5\%} conversion improvement in the first half after its initial launch, and \textbf{+1.03\%} and \textbf{+1.22\%} conversion improvement from two individual launches in the subsequent half. Through systematic experiments, LoopFM demonstrates a scaling law in sequence length, embedding dimension, and upstream FM size.}

\date{\today}
\correspondence{Huayu Li at \email{huayuli@meta.com}}

\begin{document}

\raggedbottom
\hbadness=4000
\maketitle

\section{Introduction}
\label{sec:intro}

Industrial recommendation relies on a two-tier architecture: a large foundation model (FM) with up to trillions of parameters learns rich representations offline, while compact vertical models (VMs) serve predictions under strict latency constraints~\citep{he2014practical, liang2025exfm, anil2022factory}.
Knowledge distillation (KD)~\citep{hinton2015distilling} bridges the two by transferring the FM's scalar predictions as soft labels to supervise VM training.
However, as FMs scale to multi-trillion parameters, we observe that the \emph{transfer ratio} (TR)---the fraction of FM improvement captured by the VM ($\text{TR} = \Delta\text{NE}_{\text{VM}} / \Delta\text{NE}_{\text{FM}}$)---continues to deteriorate, corroborating the findings that KD degrades under large teacher-student capacity gaps~\citep{cho2019efficacy}.

We hypothesize that this deterioration stems from a \emph{bandwidth bottleneck}: a single scalar prediction compresses all of the FM's learned knowledge---rich cross-domain features, multi-level interaction patterns, contextual signals---into one number.
As FMs grow more capable, the gap between what they learn and what a scalar can convey widens.
Furthermore, FMs are typically trained on richer cross-domain features than VMs, creating a \emph{feature gap} that scalar KD cannot bridge.

We propose \textbf{LoopFM}, a novel knowledge transfer framework to address this bottleneck by \emph{materializing} the FM's intermediate representations as structured input features for VM consumption (Figure~\ref{fig:overview}).
LoopFM defines three modular stages---extraction, compression, and structuring---each independently configurable.
The structuring stage groups compressed embeddings by a \emph{grouping key} (e.g., user ID, item ID), yielding, e.g., sequences or graphs, as VM input features.
In this paper we focus on user-keyed temporal sequences.
Crucially, only \emph{historical} embeddings are used---enabling VM serving without requiring real-time FM inference.

We provide a theoretical analysis (Section~\ref{sec:theory}) showing that LoopFM's information gain decomposes into temporal history and cross-feature components minus the compression cost, and derive a lower bound on the transfer ratio that increases monotonically with the FM's feature gap over the VM.

We validate LoopFM on public benchmarks (TaobaoAd, KuaiVideo, Amazon Electronics) and industrial-scale systems with trillion-parameter FMs, where it doubles the knowledge transfer ratio and has delivered significant ad conversion improvement in production. Our \textbf{contributions} are:
\begin{enumerate}[leftmargin=2em, itemsep=1pt, topsep=2pt]
    \item A modular framework that transfers FM intermediate representations as structured VM input features, requiring no architectural coupling or real-time FM inference.
    \item Theoretical analysis decomposing LoopFM's information gain into temporal, cross-feature, and compression-loss components, with a transfer-ratio lower bound that grows with the FM-VM feature gap (Corollary~\ref{cor:tr-monotone}) and tightens with better compression (Theorem~\ref{thm:loopfm-tr}), and an information gain that non-decreases with sequence length $L$ (Theorem~\ref{prop:monotone-L-main}).
    \item Extensive experiments on three public benchmarks and internal trillion-parameter systems, including ablations on layer selection, checkpoint freshness, sequence length, embedding dimension etc. In production, LoopFM delivered +0.5\% conversion improvement in the first deployment half, and +1.03\% and +1.22\% conversion improvements from two individual launches in the subsequent half.
    \item Knowledge-transfer scaling: LoopFM's gains increase with sequence length and FM scale, and with embedding dimension when the downstream encoder has sufficient capacity, transferring improvements from richer FM features and greater FM capacity to the VM.
\end{enumerate}

\begin{figure}[t]
    \centering
    \begin{tikzpicture}
    \node[inner sep=0pt, anchor=west] (img) at (0,0) {\includegraphics[width=0.62\textwidth]{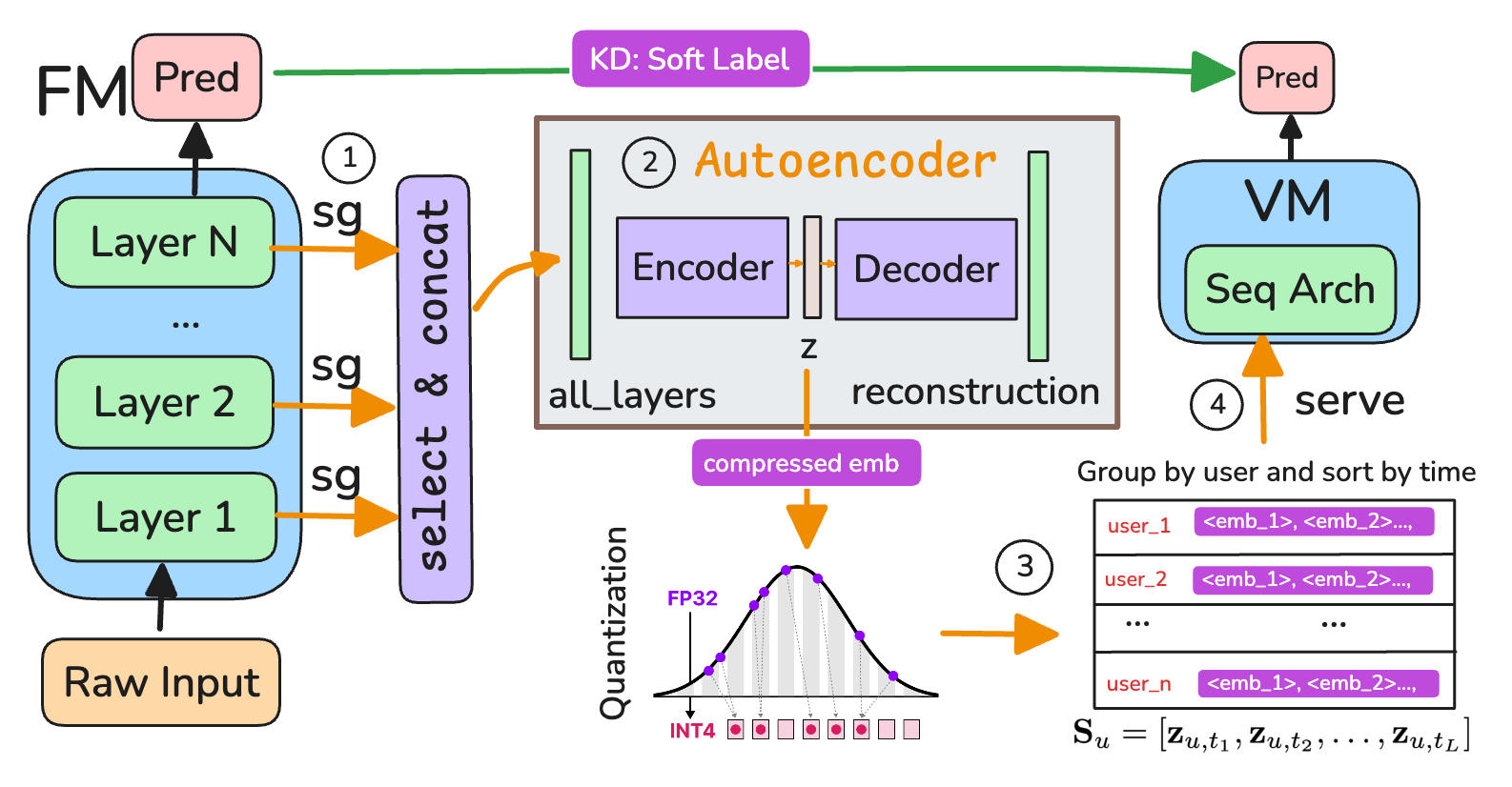}};

    \node[draw=orange!70, fill=orange!5, rounded corners=4pt, text width=4.55cm, font=\scriptsize,
          align=left, anchor=south west] (prod) at ([xshift=4pt, yshift=2pt]img.east) {%
        \textbf{\color{orange!70!black}Internal (Production)}\\[2pt]
        * Transfer ratio: $\sim$\textbf{doubled}\\[1pt]
        * Conversions (Since initial launch):\\
        \quad 1st Half: \textbf{+0.50\%}\\
        \quad 2nd Half: \textbf{+1.03\%} and \textbf{+1.22\%}\\[1pt]
    };

    \node[draw=blue!60, fill=blue!5, rounded corners=4pt, text width=4.55cm, font=\scriptsize,
          align=left, anchor=north west] (pub) at ([xshift=4pt, yshift=-8pt]img.east) {%
        \textbf{\color{blue!70!black}Public Benchmarks AUC gain}\\[2pt]
        * TaobaoAd: \textbf{+6.4\%} avg (+6.1--6.6\%)\\
        * KuaiVideo: \textbf{+1.0\%} avg (+0.6--1.6\%)\\
        * Amazon: \textbf{+0.5\%} avg (+0.02--1.14\%) \\
    };
    \end{tikzpicture}
    \caption{LoopFM overview. \textbf{Left:} KD transfers via a soft label (top). LoopFM adds a high-bandwidth embedding channel: \ding{192}~extract intermediate FM embeddings; \ding{193}~compress (e.g., autoencoder) and quantize (e.g., INT4) for storage efficiency; \ding{194}~group by key and structure into temporal sequences $\mathbf{S}_k$ (here, user-keyed: $k{=}u$); \ding{195}~serve as input features to the VM's encoder. No real-time FM inference is required. \textbf{Right:} Key results on internal production (top) and public benchmarks (bottom).
    }
    \label{fig:overview}
\end{figure}

\section{Related Work}
\label{sec:related}

\paragraph{Knowledge distillation in recommendation.}
KD~\citep{hinton2015distilling} is widely used to transfer knowledge from large teachers to compact students in recommendation~\citep{kang2024unbiased, kang2023distillation, chen2023unbiased, tang2018ranking}.
External distillation~\citep{liang2025exfm, khani2024bridging} separates teacher and student training, enabling the teacher to serve as a foundation model.
Beyond scalar KD, embedding-based methods transfer intermediate representations: FitNets~\citep{romero2015fitnets} matches hidden activations, CRD~\citep{tian2020crd} maximizes representation mutual information, and privileged features distillation~\citep{xu2020privileged, yang2022privileged} leverages training-time-only signals.
\citet{cui2024dllm2rec} distill LLM embeddings to lightweight sequential models.
All these methods either transfer scalars or match \emph{current-sample} representations via auxiliary losses.
LoopFM instead materializes \emph{historical} FM embeddings as structured input features.

\paragraph{Sequence modeling and FM representations.}
User behavior sequence modeling is central to modern recommendation: DIN~\citep{zhou2018din} and DIEN~\citep{zhou2019dien} use attention over action histories, SASRec~\citep{kang2018sasrec} and BERT4Rec~\citep{sun2019bert4rec} apply Transformers, and SIM~\citep{pi2020sim} handles lifelong sequences.
These model raw behavioral signals (clicked IDs, categories).
Separately, foundation model research~\citep{hou2026kunlun, zhang2024wukong, anil2022factory} has explored transferable entity embeddings: PinSage~\citep{ying2018graph} for items, entity-level user embeddings~\citep{zhang2024sum, li2023catart} for users, which miss interaction-level information.
Recent industrial systems~\citep{xiong2026latte, chen2025vista} have adopted asynchronous caching of sequence model representations for serving efficiency.
Concurrent and independent of our work, IAT~\citep{li2026iat} proposes a similar compress-and-sequence pipeline for historical interaction representations from a sequence feature engineering perspective.
While the high-level idea is shared, LoopFM is a general framework motivated by FM-to-VM transfer ratio improvement, with theoretical insights.

\section{Method}
\label{sec:method}

\paragraph{Overview and problem formulation.}
\label{sec:overview}
We consider a standard industrial recommendation setup: a large foundation model (FM) with trillions of parameters and multiple compact vertical models (VMs) with millions of parameters for serving.
The FM is trained on aggregated cross-domain data while each VM handles a specific ranking stage or domain.
In standard external KD~\citep{liang2025exfm}, the FM generates scalar predictions $\hat{y}^F$ as soft labels for VM training.
LoopFM opens a \emph{high-bandwidth embedding channel} alongside this scalar channel, through three modular stages---\emph{extraction}, \emph{compression}, and \emph{structuring}---each admitting different algorithmic choices.

\paragraph{Stage 1: Embedding extraction.}
\label{sec:extraction}
Given an FM with layers $l_1, \ldots, l_M$, we select a subset of $K$ layers and concatenate their activations to form a raw embedding $\mathbf{e}^{(i)} = [\mathbf{h}_{l_1}^{(i)}; \ldots; \mathbf{h}_{l_K}^{(i)}] \in \mathbb{R}^{D}$ for each example (\emph{e.g.}, an item impression $a$ for a user $u$: $(u, a)$), where $D = \sum_{k=1}^{K} d_{l_k}$ ($D$ could be $O(10^5)$--$O(10^6)$ for industrial FMs). One principle is that the selected layers should be shallow enough to retain rich input information~\citep{tishby2015deep}, but also deep enough to capture well-learned interactions.
Extraction shares inference pass with KD, adding negligible overhead.

\paragraph{Stage 2: Compression.}
\label{sec:compression}
The raw embeddings $\mathbf{e}^{(i)} \in \mathbb{R}^D$ are too high-dimensional for direct storage.
We compress them to $\mathbf{z}^{(i)} \in \mathbb{R}^d$ ($d \ll D$) using an autoencoder co-trained with the FM, though any dimensionality reduction method (PCA, random projection, learned codebook) could serve this role:
\begin{equation}
    \mathbf{z}^{(i)} = f_{\text{enc}}(\mathbf{e}^{(i)}), \quad \hat{\mathbf{e}}^{(i)} = f_{\text{dec}}(\mathbf{z}^{(i)}), \quad \mathcal{L}_{\text{AE}} = \|\mathbf{e}^{(i)} - \hat{\mathbf{e}}^{(i)}\|_2^2.
\end{equation}
A stop-gradient ensures no autoencoder gradients flow into the FM backbone.
The encoder's last layer uses $\tanh$ activation to bound outputs in $[-1, 1]$, facilitating INT4 quantization ($z_{\text{quant}} = \text{round}(z \cdot 8).\text{clamp}(-8, 7)/8$), which achieves 4$\times$ storage reduction from FP16.

\paragraph{Matryoshka compression.}
To enable flexible embedding dimension tuning for deployment allowing trade-off of feature value vs. storage/latency cost, we use Matryoshka-style autoencoders~\citep{kusupati2022matryoshka} so that any prefix $\mathbf{z}_{1:d'}$ is a valid representation, by summing losses over target dimensions $\mathcal{D}$: $\mathcal{L}_{\text{MAE}} = \sum_{d' \in \mathcal{D}} \|\mathbf{e} - f_{\text{dec}}^{(d')}(\mathbf{z}_{1:d'})\|_2^2$.

\paragraph{Stage 3: Structuring.}
\label{sec:sequence}
Historical compressed embeddings within a configurable time window are first grouped by a \emph{key} $k$ (e.g., user ID, item ID, or other semantic IDs), and then stored in appropriate structures, such as temporal sequences:
\begin{equation}
    \mathbf{S}_k = [\mathbf{z}_{k,t_1}, \mathbf{z}_{k,t_2}, \ldots, \mathbf{z}_{k,t_L}], \quad t_L < \cdots < t_2 < t_1 < t_{cur},
\end{equation}
where $L$ is the maximum sequence length, $t_i$ is the timestamp of event $i$, and $t_{cur}$ is current timestamp when the user is being served.
Note we \emph{exclude} the current sample being served to avoid requiring real-time FM inference at serving.
More generally, the grouped embeddings can be structured as a graph (e.g., a user-item bipartite graph with FM embeddings on edges).
In this paper, we focus on user-keyed sequences ($k = u$): for each user $u$, we collect their historical FM embeddings within a retention window (e.g., 30 days) and truncate to the most recent $L$ entries (e.g., $200$) to form $\mathbf{S}_u$.

\paragraph{VM-side integration.}
\label{sec:downstream}
The VM processes $\mathbf{S}_u$ via a sequence encoder (denoted ``Seq Arch'' in Figure~\ref{fig:overview}), such as mean/sum pooling or attention, and concatenates the pooled representation with other feature embeddings before the interaction layers.
The training loss combines task loss and KD:
$\mathcal{L} = \mathcal{L}_{\text{task}}(\hat{y}^V, y) + \lambda \cdot \mathcal{L}_{\text{KD}}(\hat{y}^V, \hat{y}^F)$,
where $\hat{y}^V = g(\mathbf{x}_{\text{VM}}, \mathbf{S}_u; \Theta_V)$.


\section{Theoretical Analysis}
\label{sec:theory}

We formalize the sources of LoopFM gain and analyze how FM improvement transfers to the VM.

\paragraph{Setup.}
Although the LoopFM framework applies to any grouping key (Section~\ref{sec:sequence}), we present the theoretical analysis using user ID key without loss of generality---all results hold verbatim by replacing ``user'' with any key type.
Consider predicting label $y$ given current VM features $\mathbf{x}_{\text{VM}}^{(t)}$ for a user $u$ with $L$ past interactions. The FM observes a superset $\mathbf{x}_{\text{FM}} = (\mathbf{x}_{\text{VM}}, \mathbf{x}_{\text{extra}})$ including cross-domain signals $\mathbf{x}_{\text{extra}}$ unavailable to the VM. We denote the user's raw VM-side history (past ad IDs, categories, etc.) as $\mathbf{H}_u := (\mathbf{x}_{\text{VM}}^{(t_1)}, \ldots, \mathbf{x}_{\text{VM}}^{(t_L)})$, and the LoopFM sequence of compressed FM$_k$ embeddings as $\mathbf{S}_u^{(k)} := [\mathbf{z}_{u,t_1}^{(k)}, \ldots, \mathbf{z}_{u,t_L}^{(k)}]$, where each entry encodes the FM's processing of \emph{all} features at that timestep. We compare three configurations: (i)~\textbf{FM} (full features), (ii)~\textbf{Baseline VM} (current features $\mathbf{x}_{\text{VM}}^{(t)}$ only), and (iii)~\textbf{LoopFM VM} (augmented with $\mathbf{S}_u^{(k)}$). For each, the Bayes risk is $\mathcal{R}^*(\cdot) = H(y \mid \cdot)$ and the achieved risk is $R_{\text{ach}}(\cdot)$.

\subsection{Gain Decomposition}
\label{sec:gain-decomp}

\begin{theorem}[Gain decomposition]\label{thm:gain-decomp} Let $I(\cdot\ ;\ \cdot|\ \cdot)$ denote the conditional mutual information. The LoopFM information gain $\mathcal{I}_{\text{LoopFM}}(\text{FM}_k) := \mathcal{R}^{*(\text{VM})} - \mathcal{R}^{*(\text{LoopFM}_k)}$ decomposes exactly as:
\begin{equation}\label{eq:gain-decomp}
    \mathcal{I}_{\text{LoopFM}}(\text{FM}_k)
    \;=\; \underbrace{\mathcal{I}_{\mathrm{temporal}}}_{\substack{\text{user event} \\ \text{interactions}}}
    \;+\; \underbrace{\mathcal{I}_{\mathrm{cross},k}}_{\substack{\text{feature-gap} \\ \text{gain}}}
    \;-\; \underbrace{\mathcal{I}_{\mathrm{residual},k}}_{\substack{\text{compression} \\ \text{loss}}},
\end{equation}
where all three terms are non-negative conditional mutual information:
\begin{enumerate}[label=(\roman*), itemsep=2pt, topsep=3pt]
    \item $\mathcal{I}_{\mathrm{temporal}} := I(\mathbf{H}_u; y \mid \mathbf{x}_{\text{VM}}^{(t)})$: predictive information in user history beyond current features---how much past behavior helps predict $y$, independent of the FM;
    \item $\mathcal{I}_{\mathrm{cross},k} := I(\mathbf{S}_u^{(k)}; y \mid \mathbf{x}_{\text{VM}}^{(t)}, \mathbf{H}_u)$: additional information from the FM's extra features $\mathbf{x}_{\text{extra}}$ as encoded in the LoopFM embeddings;
    \item $\mathcal{I}_{\mathrm{residual},k} := I(\mathbf{H}_u; y \mid \mathbf{x}_{\text{VM}}^{(t)}, \mathbf{S}_u^{(k)})$: information in $\mathbf{H}_u$ \emph{not} captured by $\mathbf{S}_u^{(k)}$---the price of compression: both DNN and autoencoder lead to loss of mutual info with input.
\end{enumerate}
LoopFM's gain thus comes from two sources---temporal history and cross-feature information---minus the compression cost. By the data processing inequality, 
$$\mathcal{I}_{\mathrm{cross},k} \le I(\mathbf{x}_{\text{extra},k}^{(t_1:t_L)}; y \mid \mathbf{x}_{\text{VM}}^{(t)}, \mathbf{H}_u) =: \mathcal{I}_{\mathrm{feature\text{-}raw},k}.$$ See Appendix~\ref{app:proof-gain-decomp} for the formal theorem statement and proof.
\end{theorem}


\paragraph{Pipeline loss fractions.}
The compression loss $\mathcal{I}_{\mathrm{residual},k}$ further decomposes along the pipeline (FM $\to$ autoencoder $\to$ quantization) into three non-negative terms (Proposition~\ref{prop:pipeline-decomp} in Appendix~\ref{app:pipeline-decomp}), defining two retention parameters:
\begin{itemize}[leftmargin=1.5em, itemsep=1pt, topsep=2pt]
    \item $\tau_k \ge 0$: the \emph{temporal} pipeline loss fraction, satisfying $\mathcal{I}_{\mathrm{residual},k} \le \tau_k \cdot \mathcal{I}_{\mathrm{temporal}}$;
    \item $\eta_k \in [0,1]$: the \emph{cross-platform} pipeline loss fraction, satisfying $\mathcal{I}_{\mathrm{cross},k} = (1 - \eta_k) \cdot \mathcal{I}_{\mathrm{feature\text{-}raw},k}$.
\end{itemize}
Both decrease with larger FM (reducing representation loss), larger autoencoder dimension $d$ (reducing compression loss), and finer quantization; see Appendix~\ref{app:pipeline-decomp} for detailed notation and the full pipeline decomposition.
Substituting into~\eqref{eq:gain-decomp} yields a two-sided ``gain sandwich'' (Corollary~\ref{cor:gain-lower-bound} in Appendix~\ref{app:pipeline-decomp}) that bounds $\mathcal{I}_{\text{LoopFM}}$ in terms of $\tau_k$ and $\eta_k$.

\subsection{Transfer Ratio Analysis}
\label{sec:transfer-ratio}


\paragraph{Two teacher setup.}
Consider two overparameterized foundation models:
FM$_1$ (old) with features $\mathbf{x}_{\text{FM}_1} = (\mathbf{x}_{\text{VM}}, \mathbf{x}_{\text{extra},1}) \in \mathbb{R}^{m_1}$ and $p_1 \gg m_1$ parameters, and
FM$_2$ (new) with features $\mathbf{x}_{\text{FM}_2} = (\mathbf{x}_{\text{VM}}, \mathbf{x}_{\text{extra},1}, \mathbf{x}_{\text{extra},2}) \in \mathbb{R}^{m_2}$ and $p_2 \ge p_1$, $p_2 \gg m_2$ parameters.
The VM observes $\mathbf{x}_{\text{VM}} \in \mathbb{R}^{m_s}$ with $m_s < m_1 < m_2$.
Each FM's achieved risk decomposes as $R_{\text{ach}}(\text{FM}_k) = \mathcal{R}^{*(\text{FM}_k)} + \epsilon_{\text{over}}(p_k, m_k)$, where $\epsilon_{\text{over}}$ is the excess risk.
The total FM improvement decomposes as:
\begin{equation}\label{eq:teacher-decomp}
    \Delta_{\text{teacher}} := R_{\text{ach}}(\text{FM}_1) - R_{\text{ach}}(\text{FM}_2)
    = \underbrace{\mathcal{R}^{*(\text{FM}_1)} - \mathcal{R}^{*(\text{FM}_2)}}_{\Delta_{\text{feat}} \;\ge\; 0}
    + \underbrace{\epsilon_{\text{over}}(p_1, m_1) - \epsilon_{\text{over}}(p_2, m_2)}_{\Delta_{\text{param}}},
\end{equation}
where $\Delta_{\text{feat}}$ is the Bayes-risk gain from richer features and $\Delta_{\text{param}}$ captures model overparameterization effects.
The \emph{LoopFM transfer ratio} is defined as 
\begin{equation}\label{eq:tr-def}
    \mathrm{TR}_{\text{LoopFM}} := \frac{\Delta_{\text{LoopFM}}}{\Delta_{\text{teacher}}}. 
\end{equation}
where $\Delta_{\text{LoopFM}} :=
  R_{\text{ach}}(\text{LoopFM}_1) - R_{\text{ach}}(\text{LoopFM}_2)$ denotes the LoopFM-driven VM improvement.

\paragraph{Assumptions (simplified; formal versions in Appendix~\ref{app:assumptions}).}
We make three assumptions:
\begin{itemize}[leftmargin=2em, itemsep=2pt, topsep=3pt]
    \item[\textbf{(A1)}] \textbf{Overparameterized FM with controlled excess risk.} For each FM$_k$, let $\epsilon_{\text{over}}(p_k,m_k) := R_{\text{ach}}(\text{FM}_k) - \mathcal{R}^{*(\text{FM}_k)}$ denote the expected excess risk ($n$ training samples). We assume the NTK/lazy-training regime~\citep{jacot2018neural} and the benign-overfitting conditions of \citet[Theorem~1]{bartlett2020benign} at split index $m_k$, yielding a two-sided bound:
    \begin{equation}\label{eq:a1-envelope-main}
    \underline{C}_{\text{over}}\,\sigma^2\,\xi_k \;\le\; \epsilon_{\text{over}}(p_k, m_k) \;\le\; \bar{C}_{\text{over}}\,\sigma^2\,\xi_k,
    \end{equation}
    where $\xi_k := m_k/n + n/(p_k - m_k) \to 0$ as $p_k, n \to \infty$ with $m_k = o(n)$ and $p_k \gg n$, and $0 < \underline{C}_{\text{over}} \le \bar{C}_{\text{over}}$ are constants; see derivation in Appendix~\ref{app:overparameterization}.
    \item[\textbf{(A2)}] \textbf{Informative new features.} For additional FM$_2$ features $\mathbf{x}_{\text{extra},2} = (u_{m_1+1},\ldots,u_{m_2})$, there exist constants $0<\underline{\kappa}_{\mathrm{gap}}\le \bar{\kappa}_{\mathrm{gap}}$ and $0<\underline{\kappa}_{\mathrm{gap}}^{\mathrm{hist}}\le \bar{\kappa}_{\mathrm{gap}}^{\mathrm{hist}}$ such that for each $j\in\{m_1{+}1,\ldots,m_2\}$, we have the following information bounds:
    \begin{itemize}
        \item[(a)]\emph{current information:} $\underline{\kappa}_{\mathrm{gap}}
    \le
    I\!\left(u_j^{(t)}; y \mid \mathbf{x}_{\text{VM}}^{(t)},\mathbf{x}_{\text{extra},1}^{(t)},u_{m_1+1:j-1}^{(t)}\right)
    \le \bar{\kappa}_{\mathrm{gap}}$;
    \item[(b)] \emph{historical information:} $\underline{\kappa}_{\mathrm{gap}}^{\mathrm{hist}}
    \le
    I\!\left(u_j^{(t_1:t_L)}; y \mid \mathbf{x}_{\text{VM}}^{(t)},\mathbf{H}_u,\mathbf{x}_{\text{extra},1}^{(t_1:t_L)},u_{m_1+1:j-1}^{(t_1:t_L)}\right)
    \le \bar{\kappa}_{\mathrm{gap}}^{\mathrm{hist}}$.
    \end{itemize}

    \item[\textbf{(A3)}] \textbf{Pipeline quality.} The total cross-platform information lost by FM$_2$'s pipeline (representation + autoencoder + quantization) is no worse than FM$_1$'s:
    $\ell_{\mathrm{repr},2}^{\mathrm{cross}} + \ell_{\mathrm{AE},2}^{\mathrm{cross}} + \ell_{\mathrm{Q},2}^{\mathrm{cross}} \le \ell_{\mathrm{repr},1}^{\mathrm{cross}} + \ell_{\mathrm{AE},1}^{\mathrm{cross}} + \ell_{\mathrm{Q},1}^{\mathrm{cross}}$,
    where $\ell_{\mathrm{repr},k}^{\mathrm{cross}}$, $\ell_{\mathrm{AE},k}^{\mathrm{cross}}$, and $\ell_{\mathrm{Q},k}^{\mathrm{cross}}$ are the MI losses from FM representation, autoencoder, and quantization on the cross-platform channel (Appendix~\ref{app:pipeline-decomp}). It holds when $p_2 \ge p_1$ (FM capacity), $d_2 \ge d_1$ (AE bottleneck), and $b_2 \ge b_1$ (quantization bits).
\end{itemize}

\begin{theorem}[Transfer-ratio bound]\label{thm:loopfm-tr}
Assume \emph{(A1)--(A2)}, well-trained students ($\epsilon_{\mathrm{est},k} \to 0$), and $\Delta_{\text{teacher}} > 0$. Then it holds that

\textbf{(0) Initial launch.} When the VM has no prior LoopFM ($R_{\text{ach}}(\text{LoopFM}_1) = R_{\text{ach}}(\text{VM})$), $\mathrm{TR}_{\text{LoopFM}} \ge 0$ since adding $\mathbf{S}_u^{(2)}$ can only reduce Bayes risk. Corollary~\ref{cor:gain-lower-bound} gives a quantitative bound:
\begin{equation}\label{eq:tr-launch}
\mathrm{TR}_{\text{LoopFM}}
\;\ge\;
\frac{(1 - \tau_2)\,\mathcal{I}_{\mathrm{temporal}} + (1 - \eta_2)\,\mathcal{I}_{\mathrm{feature\text{-}raw},2}}{\Delta_{\text{teacher}}}.
\end{equation}

\textbf{(1) Negative transfer (without A3).} If the new FM's pipeline is worse ($\eta_2 > \eta_1$, violating A3), negative transfer can occur: specifically when $\tau_1 \mathcal{I}_{\mathrm{temporal}} + (1{-}\eta_2)\mathcal{I}_{\mathrm{feature\text{-}raw},2} - (1{-}\eta_1)\mathcal{I}_{\mathrm{feature\text{-}raw},1} < 0$. Under \emph{(A3)}, this condition is never satisfied (Appendix~\ref{app:proof-tr}).

\textbf{(2) Positive transfer (with A3).} Let $\delta := m_2 - m_1$ denote the feature gap. Under (A3), Lemma~\ref{lem:eta-monotone} gives $\eta_2 \le \eta_1$. Then it holds that
\begin{equation}\label{eq:loopfm-tr-lb}
\mathrm{TR}_{\text{LoopFM}}
\;\ge\;
\frac{ - \tau_2\,\mathcal{I}_{\mathrm{temporal}} + (1 - \eta_1)\, \underline{\kappa}_{\mathrm{gap}}^{\mathrm{hist}}\,\delta}
{\bar{\kappa}_{\mathrm{gap}}\,\delta + \bar\kappa_{\mathrm{over}}\,\xi_1 - \underline\kappa_{\mathrm{over}}\,\xi_2}=: \mathrm{TR}_{\text{LB}}(\delta),
\end{equation}
where $\bar\kappa_{\mathrm{over}} := \bar C_{\mathrm{over}} \sigma^2$ and 
$\underline\kappa_{\mathrm{over}} := \underline C_{\mathrm{over}} \sigma^2$ arise from the two-sided (A1) envelope, and $\xi_k := m_k/n + n/(p_k - m_k)$; see detailed proof in Appendix~\ref{app:proof-tr}.
\end{theorem}

\begin{remark}[Relaxing A1 covariance structure assumption]
The bound~\eqref{eq:loopfm-tr-lb} uses $\xi_k = m_k/n + n/(p_k - m_k)$, which assumes a bi-level covariance structure with tail effective rank $R_{m_k} = p_k - m_k$.
This can be relaxed: Theorem~\ref{thm:loopfm-tr-general} (Appendix~\ref{app:general-tr-bound}) gives the same bound with $\xi_k = d_k/n + n/R_{d_k}(\Sigma_k)$ for a general split index $d_k$ and tail effective rank $R_{d_k}(\Sigma_k)$, valid under any benign covariance structure \citep[Definition 4]{bartlett2020benign}.
\end{remark}


\begin{corollary}[Monotonicity in feature gap]\label{cor:tr-monotone}
Under the conditions of Theorem~\ref{thm:loopfm-tr}(2), $\mathrm{TR}_{\text{LB}}(\delta)$ is monotonically increasing in $\delta$ for all $\delta \ge 0$, and converges to $(1{-}\eta_1)\,\underline{\kappa}_{\mathrm{gap}}^{\mathrm{hist}} / \bar{\kappa}_{\mathrm{gap}} > 0$ as $\delta \to \infty$.
\end{corollary}

\begin{theorem}[Sequence length; informal]\label{prop:monotone-L-main}
The LoopFM information gain $\mathcal{I}_{\text{LoopFM},k}(L) := I(\mathbf{S}_u^{(k,L)}; y \mid \mathbf{x}_{\text{VM}}^{(t)})$ is non-decreasing in $L$ and converges to a constant $\mathcal{I}_{\text{LoopFM},k}^* \le H(y \mid \mathbf{x}_{\text{VM}}^{(t)})$.
\end{theorem}
\noindent The formal statement and proof appear in Appendix~\ref{app:seq-length-theory} (Proposition~\ref{prop:monotone-L}).

\paragraph{Implications.}
\textbf{(1)~Pipeline quality controls transfer:} reducing $\tau_k$ and $\eta_k$---via larger AE dimension~$d$, finer quantization~$b$, or stronger FM---directly raises the transfer-ratio bound~\eqref{eq:loopfm-tr-lb}. When pipeline losses dominate, the transfer ratio becomes smaller or even turns negative: the student worsens despite a better teacher (negative transfer). This motivates careful autoencoder design and the Matryoshka multi-granularity approach.
\textbf{(2)~Impact of feature gap $\delta$:} Corollary~\ref{cor:tr-monotone} shows that a larger feature gap $\delta = m_2 - m_1$ between FM generations monotonically increases the transfer-ratio lower bound, with diminishing marginal returns as $\mathrm{TR}_{\text{LB}}$ saturates at $(1{-}\eta_1)\underline{\kappa}_{\mathrm{gap}}^{\mathrm{hist}}/\bar{\kappa}_{\mathrm{gap}}$.
This means LoopFM benefits most from FM upgrades that introduce many new features, and the pipeline loss $\eta_1$ is the dominant bottleneck limiting how much of this benefit is realized---motivating the \emph{Matryoshka autoencoder} for flexible compression.
\textbf{(3)~Diminishing returns in sequence length:} Theorem~\ref{prop:monotone-L-main} guarantees monotone gains in~$L$ with diminishing marginal returns (since the total is bounded by $H(y)$), consistent with Table~\ref{tab:ablation}.
\textbf{(4)~FM capacity saturation and TR decay:} scaling FM$_2$'s parameters $p_2$ improves both the denominator $\Delta_{\text{teacher}}$ (via reduced $\epsilon_{\text{over}}(p_2,m_2)$) and the numerator $\Delta_{\text{LoopFM}}$ (via reduced $\tau_2$ as representation loss decreases), but the numerator saturates once representation loss plateau ($\ell_{\mathrm{repr},k}\to0$; see Appendix~\ref{app:pipeline-decomp} for definition), while the denominator continues growing as $\epsilon_{\text{over}}(p_2,m_2)\to 0$. Consequently, TR can \emph{decrease} with FM scaling---consistent with Table~\ref{tab:fm_capacity} (AUC spread $< 0.001$ across $16\times$ FM capacity).

\section{Experiments}
\label{sec:experiments}

\subsection{Experiment Setup}
\label{sec:setup}

\paragraph{Public datasets.}
We use three public datasets: \textbf{TaobaoAd}~\citep{tianchi2018taobao} (25M samples, 22 features), \textbf{KuaiVideo}~\citep{kuaishou} (13.7M samples, 9 features), and \textbf{Amazon Electronics}~\citep{he2016ups} (3M samples, 6 features).
We present main results and ablations on TaobaoAd (the richest feature set); results on KuaiVideo and Amazon Electronics are in Appendix~\ref{app:kuaivideo} and~\ref{app:amazon}.
TaobaoAd is partitioned into 8 calendar days, and we follow a temporal streaming protocol by~\citet{liang2025exfm} and a typical VM training setup for Ads CTR models: FM (DMIN~\citep{xiao2020deep}) is trained on days 1--4 and evaluated on days 5--8 to log labels and embeddings for KD and LoopFM sequence building; VM uses \emph{one-pass} streaming (no shuffle) training on days 5--7 (no validation for early stopping needed), and we report test metrics on day 8.

\paragraph{Models.}
The FM is DMIN~\citep{xiao2020deep} with BARS-optimal hyperparameters (embedding dim 32, DNN [512, 256, 128]) and a co-trained autoencoder ($d{=}32$).
We evaluate six VMs: non-sequential (\underline{FM}, Fm\underline{FM}, Deep\underline{FM}~\citep{guo2017deepfm}) and sequential (DIEN~\citep{zhou2019dien}, DMR~\citep{lyu2020deep}, DMIN), each extended (if needed) with a DMIN-style sequence module for LoopFM ($d{=}32$, $L{=}50$).
KD uses the auxiliary head framework from~\citet{liang2025exfm} ($\alpha{=}5$, $\gamma{=}10$, $\beta{=}5$ on TaobaoAd; tuned per-dataset in appendix).
Internal experiments use trillion-parameter FMs with billions of daily training examples and INT4 quantization for storage efficiency; public benchmark experiments use FP32 embeddings without quantization.
Our public benchmark implementation is built on FuxiCTR~\citep{zhu2022bars, zhu2021open}.

\paragraph{Metrics.}
\textbf{NE} (Normalized Entropy)~\citep{he2014practical}: normalized LogLoss (lower is better). \textbf{AUC} (higher is better). \textbf{Transfer ratio}: $\text{TR} = \Delta\text{NE}_{\text{VM}} / \Delta\text{NE}_{\text{FM}}$. ``KD'' refers to auxiliary-head-based external KD~\citep{liang2025exfm} throughout.

\subsection{Research Questions}
\label{sec:rqs}

\begin{itemize}[leftmargin=3.2em, labelsep=0.5em, labelwidth=2.5em, align=left, itemsep=1pt, topsep=2pt]
    \item[\textbf{RQ1}] Does LoopFM improve over scalar KD? (Section~\ref{sec:rq1})
    \item[\textbf{RQ2}] How does LoopFM compare to other embedding-based transfer methods? (Section~\ref{sec:rq2})
    \item[\textbf{RQ3}] Does sequence structure matter, or is pooling sufficient? (Section~\ref{sec:rq3})
    \item[\textbf{RQ4}] Are FM-learned sequences better than hand-crafted raw ID sequences? (Section~\ref{sec:rq4})
    \item[\textbf{RQ5}] How do layer selection, checkpoint frequency, sequence length, and embedding dimension affect performance? (Section~\ref{sec:rq5})
    \item[\textbf{RQ6}] How does FM capacity affect LoopFM's benefit? (Appendix~\ref{sec:rq6})
\end{itemize}

\subsection{Public Dataset Results}
\label{sec:public}

\subsubsection{RQ1: Does LoopFM Improve over KD?}
\label{sec:rq1}

Table~\ref{tab:public_taobao} presents results on TaobaoAd across six VMs (we underline \underline{FM} for factorization machine to distinguish from FM).
LoopFM consistently improves over KD across all architectures.
When the VM architecture is identical to the FM (DMIN), KD shows little gain as expected, yet LoopFM still brings 6.52\% AUC gain---strong evidence of complementary value.
A variance study with 5 seeds on Deep\underline{FM} (Appendix~\ref{app:variance}) confirms the gains are statistically significant ($p < 0.001$, paired $t$-test); KD+LoopFM achieves $0.6342 \pm 0.0001$ AUC with the lowest variance across all methods.

\begin{table}[ht]
\caption{LoopFM on TaobaoAd across six VMs (FM: DMIN). Non-sequential (\underline{FM}, Fm\underline{FM}, Deep\underline{FM}) and sequential (DIEN, DMR, DMIN) architectures, each extended with a DMIN-style module for LoopFM.}
\label{tab:public_taobao}
\centering
\resizebox{\textwidth}{!}{%
\begin{tabular}{lcccccccccccc}
\toprule
& \multicolumn{2}{c}{\underline{FM}} & \multicolumn{2}{c}{Fm\underline{FM}} & \multicolumn{2}{c}{Deep\underline{FM}} & \multicolumn{2}{c}{DIEN} & \multicolumn{2}{c}{DMR} & \multicolumn{2}{c}{DMIN} \\
\cmidrule(lr){2-3} \cmidrule(lr){4-5} \cmidrule(lr){6-7} \cmidrule(lr){8-9} \cmidrule(lr){10-11} \cmidrule(lr){12-13}
Method & AUC & LogLoss & AUC & LogLoss & AUC & LogLoss & AUC & LogLoss & AUC & LogLoss & AUC & LogLoss \\
\midrule
w/o distill & 0.5790 & 0.1979 & 0.5833 & 0.1977 & 0.5886 & 0.1979 & 0.5945 & 0.1963 & 0.5967 & 0.1966 & 0.5932 & 0.1966 \\
KD & 0.5855 & 0.1979 & 0.5885 & 0.1990 & 0.5980 & 0.1964 & 0.5958 & 0.1966 & 0.6002 & 0.1962 & 0.6002 & 0.1961 \\
LoopFM-only & 0.6203 & 0.1953 & 0.6066 & 0.2020 & 0.6245 & 0.1955 & 0.6335 & 0.1951 & \textbf{0.6361} & 0.1935 & 0.6319 & 0.1950 \\
KD + LoopFM & \textbf{0.6205} & \textbf{0.1950} & \textbf{0.6106} & \textbf{0.1967} & \textbf{0.6344} & \textbf{0.1940} & \textbf{0.6354} & \textbf{0.1934} & 0.6356 & \textbf{0.1934} & \textbf{0.6346} & \textbf{0.1935} \\
\bottomrule
\end{tabular}%
}
\end{table}

\subsubsection{RQ2: Comparison with Embedding-Based Transfer Methods}
\label{sec:rq2}

\begin{wraptable}{r}{0.44\textwidth}
\vspace{-4pt}
\caption{Embedding transfer methods (TaobaoAd, FM: DMIN, VM: Deep\underline{FM}, all build on KD).}
\label{tab:ablation_rq2}
\centering
\footnotesize
\begin{tabular}{lcc}
\toprule
Method & AUC & LogLoss \\
\midrule
KD (baseline) & 0.5980 & 0.1964 \\
+ Current-Emb-as-Feature & 0.5985 & 0.1964 \\
+ Current-Emb-Loss (FitNets) & 0.5984 & 0.1965 \\
+ Entity-Only Embedding & 0.5983 & 0.1965 \\
+ LoopFM & \textbf{0.6344} & \textbf{0.1940} \\
\bottomrule
\end{tabular}
\vspace{-10pt}
\end{wraptable}

Table~\ref{tab:ablation_rq2} compares LoopFM against three alternative embedding-based transfer methods: using the FM's current-sample embedding as a VM input feature (requiring real-time FM inference), matching representations via a FitNets-style~\citep{romero2015fitnets} auxiliary loss, and using user-only embeddings without interaction information.
All three provide negligible improvement ($+$0.0003--0.0005 AUC), while LoopFM achieves $+$0.0364 ($+$6.09\% relative), demonstrating the unique value of \emph{historical embedding sequences}.

\subsubsection{RQ3: Sequence Structure vs.\ Pooled Embeddings}
\label{sec:rq3}

\begin{wraptable}{r}{0.45\textwidth}
\vspace{-12pt}
\caption{Sequence aggregation ablation (TaobaoAd, FM: DMIN, VM: Deep\underline{FM}, all build on KD).}
\label{tab:ablation_rq3}
\centering
\footnotesize
\begin{tabular}{lcc}
\toprule
Method & AUC & LogLoss \\
\midrule
+ Mean-Pool Historical Emb & 0.6118 & 0.1955 \\
+ Sum-Pool Historical Emb & 0.6300 & 0.1941 \\
+ LoopFM (DIN attention) & 0.6112 & 0.1954 \\
+ LoopFM (DMIN attention) & \textbf{0.6344} & \textbf{0.1940} \\
\bottomrule
\end{tabular}
\end{wraptable}

Table~\ref{tab:ablation_rq3} compares four aggregation strategies. All methods, including simple mean-pool, show significant gain over the KD baseline (0.5980).
Sum-pool is surprisingly strong, suggesting that even storing sum-pooled historical embeddings as a feature (much lower storage/modeling cost) could already capture substantial gain. DMIN attention brings further significant improvement over sum-pool. We also observe internally that LoopFM gain is more pronounced as we scale up the sequence modeling module (Appendix~\ref{app:embed-dim}).

\subsubsection{RQ4: FM-Learned Embeddings vs.\ Raw Sequences}
\label{sec:rq4}

\begin{wraptable}{r}{0.52\textwidth}
\vspace{-12pt}
\caption{Raw sequences vs.\ LoopFM at two VM capacity levels (TaobaoAd, FM: DMIN, no KD).}
\label{tab:raw_vs_loopfm}
\centering
\scriptsize
\begin{tabular}{lcc|cc}
\toprule
& \multicolumn{2}{c|}{VM = FM (emb=32)} & \multicolumn{2}{c}{Small VM (emb=2)} \\
Method & AUC & LogLoss & AUC & LogLoss \\
\midrule
DMIN Baseline & 0.5932 & 0.1966 & 0.5843 & 0.1967 \\
+ Raw User Seq & 0.6302 & 0.1942 & 0.6186 & 0.1950 \\
+ LoopFM & 0.6319 & 0.1950 & 0.6233 & 0.1955 \\
+ Both & \textbf{0.6362} & \textbf{0.1942} & \textbf{0.6305} & \textbf{0.1941} \\
\bottomrule
\end{tabular}
\vspace{-14pt}
\end{wraptable}

Table~\ref{tab:raw_vs_loopfm} compares LoopFM sequences against custom-built raw user behavior sequences (all four ad-side ID features: ad group, campaign, advertiser, bucketized price) using DMIN as VM at two capacities: matching the FM (emb=32) and $16\times$ smaller (emb=2).

When the VM is identical to the FM (left), LoopFM only slightly outperforms raw sequences ($+$0.0017 AUC, $+$0.27\% relative)---as expected, since the VM can learn equally effective embeddings directly from raw IDs.
Combining both still yields the best result (0.6362).
When the VM is much smaller (right), LoopFM's advantage widens to $+$0.0047 AUC ($+$0.76\% relative) over raw sequences: a capacity-limited VM cannot effectively learn from high-cardinality IDs with tiny embeddings, whereas LoopFM's well-learned representations are directly consumable regardless of VM capacity.
LoopFM's marginal gain on top of raw sequences is larger for the small VM ($+$0.0119) than for the large VM ($+$0.0060), confirming that LoopFM's value increases with the FM-VM capacity gap---a typical production setting.


\subsubsection{RQ5: Hyperparameter Sensitivity}
\label{sec:rq5}

Table~\ref{tab:ablation} ablates key hyperparameters.

\paragraph{FM layer selection.}
We test the embedding layer (concatenated feature lookups before any DNN processing), each of the FM's three DNN hidden layers, joint compression of all layers, soft-label only (FM prediction score as a 1-d sequence), and item-only embeddings (raw item-side feature lookups without user/context features), compressed via separately trained AEs.

Shallower layers tend to transfer better (Emb.\ layer 0.6338, Hidden-0 0.6340 vs.\ Deep 0.6326), consistent with the information bottleneck view that deeper layers progressively discard input details~\citep{tishby2015deep}.
However, the embedding layer---which contains no learned cross-feature interactions---performs slightly below Hidden-0, suggesting that interaction learning adds value. Despite only minor differences across layers, the overall pattern perfectly aligns with the layer selection principle mentioned earlier.
The soft-label alone as a 1-d sequence (0.6286) still provides strong gains, indicating that even the \emph{temporal structure} of FM predictions carries significant information.
Joint compression of all layers (0.6341) achieves the best AUC; internally, we use this approach over a heuristically chosen set spanning shallow to deep layers.

\begin{wraptable}{r}{0.44\textwidth}
\vspace{-10pt}                                                             
\caption{Hyperparameter ablation (TaobaoAd, FM: DMIN, VM: Deep\underline{FM}, KD+LoopFM). $^\dagger$Ckpt frequency uses a separately trained fresh FM per split.
  $^\ddagger$Item-side feature embeddings only (6 ad features, no user/context, no DNN), compressed via same AE architecture.}
  \label{tab:ablation}
  \centering
  \scriptsize
  \begin{tabular}{@{}llcc@{}}
  \toprule
  Settings & & AUC & LogLoss \\
  \midrule
  \multirow{7}{*}{\textit{FM layer}}
   & Emb.\ layer & 0.6338 & \textbf{0.1936} \\
   & Hidden-0 (default) & 0.6340 & 0.1938 \\
   & Hidden-1 & 0.6328 & 0.1939 \\
   & Deep & 0.6326 & 0.1938 \\
   & All (joint, $d{=}64$) & \textbf{0.6341} & 0.1939 \\
   & Soft-label ($d{=}1$) & 0.6286 & 0.1940 \\
   & Item-only emb$^{\ddagger}$ & 0.6291 & 0.1942 \\
  \cmidrule(lr){1-4}
  \multirow{2}{*}{\textit{FM ckpt freq.$^\dagger$}}
   & Fixed (default) & \textbf{0.6344} & \textbf{0.1940} \\
   & Every split & 0.6311 & 0.1945 \\
  \cmidrule(lr){1-4}
  \multirow{5}{*}{\textit{Seq.\ length $L$}}
   & $L{=}10$ & 0.6327 & 0.1944 \\
   & $L{=}25$ & 0.6338 & 0.1941 \\
   & $L{=}50$ (default) & 0.6346 & 0.1940 \\
   & $L{=}75$ & 0.6347 & \textbf{0.1937} \\
   & $L{=}100$ & \textbf{0.6352} & 0.1938 \\
  \cmidrule(lr){1-4}
  \multirow{5}{*}{\textit{Emb.\ dim $d$}}
   & $d{=}8$ & 0.6339 & 0.1941 \\
   & $d{=}16$ & 0.6337 & 0.1937 \\
   & $d{=}32$ (default) & \textbf{0.6345} & \textbf{0.1937} \\
   & $d{=}64$ & 0.6341 & 0.1939 \\
   & $d{=}128$ & 0.6342 & 0.1941 \\
  \bottomrule
  \end{tabular}
\vspace{-14pt}
\end{wraptable}
Item-only embeddings (0.6291) capture 86\% of the full interaction embedding's gain over the KD baseline, despite containing no user or context information.
On TaobaoAd's shallow FM, item identity features (adgroup\_id, campaign\_id) are already highly predictive, so the DNN's cross-feature interactions add modest value.
The gap should widen with deeper industrial FMs where the DNN contributes a larger fraction of the representation's information.
Finding the sweet spot---deep enough to capture rich interactions, but not so deep that compression discards too much---remains an open question (see Appendix~\ref{sec:rq7}).

\paragraph{FM checkpoint frequency.}
Counterintuitively, updating the FM checkpoint every split slightly \emph{degrades} performance (AUC 0.6311 vs.\ 0.6344 for a fixed checkpoint), despite better soft-labels~\citep{liang2025exfm}.
We hypothesize this is due to embedding drift: with a fixed checkpoint, embedding centroid drift is near-zero (0.000006--0.000017), so embeddings from day 5 and day 8 live in the same space and the sequence encoder can learn consistent patterns; with updated checkpoints, centroid drift is ${\sim}$1000$\times$ larger (0.013--0.021), so embeddings from different days live in drifting spaces, making it more challenging for the sequence encoder to learn. On the other hand, embeddings are expected to drift with natural data distribution---identifying the staleness threshold beyond which freshness outweighs consistency remains an open question.

\paragraph{Effect of sequence length $L$.}
Performance increases monotonically with $L$ but with diminishing returns (Table~\ref{tab:ablation}), corroborating Theorem~\ref{prop:monotone-L-main}: AUC improves from 0.6327 ($L{=}10$) to 0.6352 ($L{=}100$), with most gain captured by $L{=}50$ (0.6346).

\paragraph{Embedding dimension $d$.}
Using Matryoshka autoencoders~\citep{kusupati2022matryoshka}, performance is stable across dimensions (AUC spread 0.0008 from $d{=}8$ to $d{=}128$).
Internally, where FMs are much larger, higher dimensions unlock significant additional gain when paired with a wider sequence encoder (Appendix~\ref{app:embed-dim}).
\subsection{Internal Results: Validation at Scale}
\label{sec:internal}

We validate LoopFM at industrial systems with trillion-parameter FMs and million-parameter VMs.

\paragraph{Transfer ratio.}
Across multiple FM-VM configurations and FM updates,
adding LoopFM approximately doubles the transfer ratio of KD alone, with the combined gain consistently exceeding either channel alone---confirming that KD and LoopFM transfer largely orthogonal knowledge.
Additional ablations on embedding dimensions are in Appendix~\ref{appendix sec: internal exp}.

\paragraph{Quantization.}
INT8 quantization is near-lossless; vanilla INT4 introduces a 0.03--0.07\% NE gap (where 0.02\% is considered significant in production). K-means INT4---learning 16 non-uniform cluster centers adapted to the $\tanh$-shaped embedding distribution---reduces this gap to 0.01--0.02\%, achieving 4$\times$ compression with minimal quality loss.

\begin{wrapfigure}{r}{0.44\textwidth}
\centering
\includegraphics[width=0.44\textwidth]{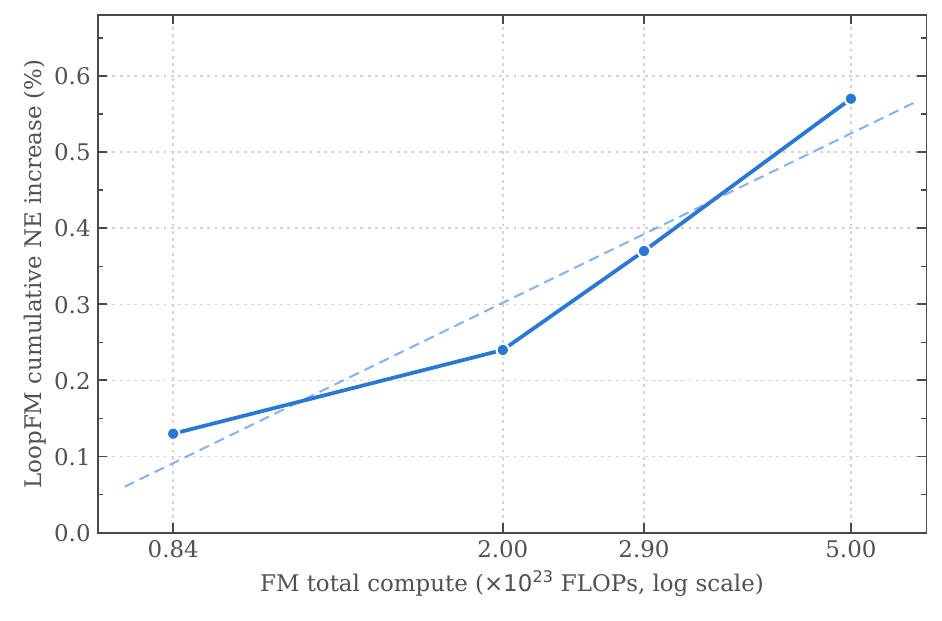}
\caption{Relative NE gain (\%) versus FM training compute (FLOPs, log scale).}
\label{fig:fm_compute_scaling}
\vspace{-\intextsep}
\end{wrapfigure}
\paragraph{Online results.}\label{subsubsec: online results}
LoopFM v1 was first launched to production on a single surface in Y1H1. Later, LoopFM v2 expanded to three additional surfaces with 4$\times$ storage cost reduction via INT4 quantization while LoopFM v3 uses the latest FM. LoopFM has improved ad conversion by +0.5\% in the first half after its initial launch, and by +1.03\% and +1.22\% from two individual launches in the subsequent half.

\paragraph{FM compute scaling.} Figure~\ref{fig:fm_compute_scaling} plots the relative NE gain of LoopFM
versus FM compute budget (measured in
FLOPs) on a logarithmic scale. The results show that LoopFM performance follows a predictable power-law relationship with FM compute.

\section{Conclusion}
\label{sec:conclusion}

We presented LoopFM, a framework that complements scalar KD with a high-bandwidth embedding channel---extracting, compressing, and structuring historical FM representations as direct VM input features.
The key insight is that \emph{what} knowledge is transferred matters as much as \emph{how}: while KD captures the FM's current prediction, LoopFM provides a general framework of structuring users' historical embeddings, encoding cross-feature patterns that a single scalar cannot convey.
Our theoretical analysis formalizes this via a gain decomposition into temporal, cross-feature and compression-loss terms, with a transfer-ratio bound that tightens as compression improves.

Empirical findings have practical implications beyond LoopFM itself.
First, our embedding-based and scalar-based transfer are largely orthogonal---combining them outperforms either alone in most settings, suggesting that multi-channel knowledge transfer should be a default design pattern in FM-VM systems.
Second, shallower FM layers transfer better than deeper ones, revealing a depth-vs-compression tradeoff: shallower representations are information-richer but contain less interaction learning, and finding the optimal extraction point remains an open problem.
Third, within the timescales we tested (2--3 days), embedding consistency within a user's sequence matters more than individual embedding freshness---though identifying the staleness threshold where this reverses remains open. Fourth, LoopFM demonstrates a predictable scaling
law.

LoopFM's three-stage design is modular: each stage can evolve independently as better compression methods, richer structuring (e.g., graph-based), or new grouping keys emerge.
Intriguing future directions include \emph{self-LoopFM}---having the FM consume its own historical embeddings to create a self-improving loop (Appendix~\ref{app:future}).

\paragraph{Limitations.}
(1)~\textbf{Storage}: LoopFM sequences require significant storage at scale.
(2)~\textbf{Cold-start}: New users lack embedding history.
(3)~\textbf{Compression}: High-dimensional embeddings inevitably lose information.
(4)~\textbf{Latency}: Sequence features incur VM inference cost.
(5)~\textbf{Scale gap}: Public benchmarks ($\sim$6M params) results may not fully generalize to trillion-parameter settings.
(6)~\textbf{Weakness of the theoretical analysis}: Our bounds are stated in terms of population-level quantities (mutual information and Bayes risk) and characterize optimal transfer under ideal learning, without accounting for finite-sample effects, optimization issues, and model capacity limitations.

\clearpage
\newpage
\bibliographystyle{assets/plainnat}
\bibliography{paper}

\clearpage
\newpage
\beginappendix

\section{Additional Internal Experiments}\label{appendix sec: internal exp}

\label{app:embed-dim}

We ablate the LoopFM embedding dimension and the sequence encoder's hidden width ($d_{\text{model}}$, i.e., the model dimension of the self-attention-based sequence encoder that processes LoopFM features) on an internal downstream model.
NE gains range from 0.09\% to 0.14\% depending on the configuration, with three key findings:
\begin{enumerate}[leftmargin=1.5em, itemsep=2pt]
  \item \textbf{A small $d_{\text{model}}$ bottlenecks transfer.}
    With $d_{\text{model}}$ fixed at the base width, increasing the embedding dimension yields no additional NE gain (all achieve 0.09\%), indicating that the sequence architecture must be wide enough to utilize the richer LoopFM information.

  \item \textbf{Embedding dimension bounds the information content.}
    Conversely, when the embedding dimension is small and fixed, enlarging $d_{\text{model}}$ provides very limited improvement (0.09\% $\to$ 0.10\%), because the LoopFM feature itself carries limited information.

  \item \textbf{Both dimensions must scale together.}
    The largest gains (0.13--0.14\%) are realized when a high embedding dimension is paired with a sufficiently expressive sequence architecture.
    Increasing $d_{\text{model}}$ from the base width to ${\sim}3\times$ improves NE gain from 0.09\% to 0.13\% at the highest embedding dimension.
\end{enumerate}

\noindent
These results are consistent with the theoretical analysis in
Section~\ref{sec:theory}: the embedding dimension controls the total
information available in the LoopFM feature (analogous to the AE
dimension~$d$ in Theorem~\ref{thm:loopfm-tr}), while $d_{\text{model}}$
governs the downstream model's capacity to absorb that information.
Corollary~\ref{cor:tr-monotone} predicts that enlarging the feature set
alone yields diminishing returns once the pipeline bottleneck (here,
$d_{\text{model}}$) saturates---precisely the plateau observed when only the embedding dimension is scaled.

\section{KuaiVideo Results}
\label{app:kuaivideo}

\begin{table}[h]
\caption{LoopFM on KuaiVideo (FM: DMIN, BARS-tuned hyperparameters~\citep{zhu2022bars}) across four VM architectures. Features: user\_id, item\_id, item\_emb (64-dim visual embedding), pos\_items/neg\_items (ID sequences), pos\_items\_emb/neg\_items\_emb (visual embedding sequences). LoopFM uses watched-only sequences (label=1), $d{=}32$, $L{=}100$.}
\label{tab:public_kuai}
\centering
\small
\begin{tabular}{lcccccccc}
\toprule
& \multicolumn{2}{c}{Deep\underline{FM}} & \multicolumn{2}{c}{DIEN} & \multicolumn{2}{c}{DMR} & \multicolumn{2}{c}{DMIN} \\
\cmidrule(lr){2-3} \cmidrule(lr){4-5} \cmidrule(lr){6-7} \cmidrule(lr){8-9}
Method & AUC & LL & AUC & LL & AUC & LL & AUC & LL \\
\midrule
w/o distill & .6877 & .5419 & .7115 & .4767 & .7158 & .4763 & .7176 & .4724 \\
KD & \textbf{.6998} & \textbf{.5125} & .7223 & .4568 & .7261 & .4528 & .7233 & .4611 \\
\midrule
LoopFM-only (mean) & .6919 & .5558 & .7158 & .4565 & .7271 & .4516 & .7261 & .4553 \\
LoopFM-only (dmin) & .6898 & .5137 & .7212 & .4542 & .7244 & .4617 & .7244 & .4600 \\
\midrule
KD + LoopFM (mean) & .6965 & .5204 & .7217 & \textbf{.4531} & \textbf{.7280} & .4527 & \textbf{.7265} & \textbf{.4527} \\
KD + LoopFM (dmin) & .6847 & .5334 & \textbf{.7252} & .4550 & .7267 & \textbf{.4512} & .7257 & .4612 \\
\bottomrule
\end{tabular}
\end{table}

Table~\ref{tab:public_kuai} presents results on KuaiVideo across four DNN-based VMs with both mean pooling and DMIN attention for the LoopFM sequence encoder.
We exclude \underline{FM} and Fm\underline{FM} as KuaiVideo's features are dominated by behavioral and visual sequences that these non-DNN models cannot process well.
KuaiVideo features include user\_id, item\_id, pretrained visual embeddings (item\_emb, 64-dim), behavioral ID sequences (pos\_items, neg\_items), and their corresponding visual embedding sequences (pos\_items\_emb, neg\_items\_emb).
Mean pooling often outperforms DMIN attention on this dataset, in contrast to TaobaoAd where DMIN attention strongly outperforms mean pooling (Table~\ref{tab:ablation_rq3}). We hypothesize that the pretrained visual embeddings dominate the LoopFM embedding space on KuaiVideo, making attention-based aggregation less effective.
We also use watched-only sequences (label=1 only), as KuaiVideo's high interaction density (${\sim}$215 per user per day) means most events are skips; including them dilutes the user-interest signal with item-specific noise.
Adding a separate unwatched sequence (analogous to neg\_items\_emb) could provide complementary negative-interest signals, but this is not the main focus of this paper.
LoopFM provides consistent AUC gains across all VMs with both pooling modes.
The best overall result is DMR KD+LoopFM (mean) at 0.7280 (+1.7\% over baseline).
For DeepFM, KD alone outperforms all KD+LoopFM variants (notably KD+LoopFM dmin drops to 0.6847, a case of negative transfer); KD+LoopFM (dmin) achieves the best AUC for DIEN (0.7252).
We note that DeepFM is not well-suited for KuaiVideo, where behavioral sequences are the dominant features---its baseline \emph{without} sequence features (0.7056, Table~\ref{tab:public_kuai_noseq}) actually outperforms the baseline \emph{with} sequences (0.6877), indicating that DeepFM struggles to leverage raw sequence features on this dataset.

Compared to TaobaoAd, LoopFM gains are smaller on KuaiVideo (+0.6--1.6\% vs.\ +6.1--6.6\%).
We attribute this primarily to lack of \emph{feature richness}: TaobaoAd has 22 features spanning user demographics, ad properties, and context, so LoopFM embeddings encode rich cross-feature interactions beyond what the raw ID-based sequences capture.
KuaiVideo has far fewer features (mostly IDs and visual embeddings), and its existing behavioral sequences already cover most of the available signal---leaving less incremental value for LoopFM to add.
This is consistent with the theoretical gain decomposition (Theorem~\ref{thm:gain-decomp}), where the cross-feature term $\mathcal{I}_{\mathrm{cross}}$ is larger when the FM processes richer features than the VM's sequences alone, as demonstrated below on no-sequence baseline.
In typical production settings, FMs train on hundreds to thousands of cross-domain features while VMs are constrained to a small feature subset due to latency requirements, creating a large feature gap that LoopFM is well-positioned to bridge.

\subsection{No-Sequence Baseline}
\label{app:kuaivideo_noseq}

\begin{table}[h]
\caption{LoopFM on KuaiVideo with no-sequence baseline (user\_id + item\_id + item\_emb only, VM: Deep\underline{FM}). Same FM, LoopFM setup (watched-only, mean pooling, $d{=}32$, $L{=}100$) as Table~\ref{tab:public_kuai}.}
\label{tab:public_kuai_noseq}
\centering
\small
\begin{tabular}{lcc}
\toprule
Method & AUC & LogLoss \\
\midrule
w/o distill & 0.7056 & 0.4710 \\
KD & 0.7071 & 0.4634 \\
LoopFM-only (mean) & \textbf{0.7167} & \textbf{0.4584} \\
KD + LoopFM (mean) & 0.7146 & \textbf{0.4584} \\
\bottomrule
\end{tabular}
\end{table}

Table~\ref{tab:public_kuai_noseq} presents results on a no-sequence baseline where the VM (Deep\underline{FM}) has access to only user\_id, item\_id, and item\_emb---no behavioral sequence features.
Note that without sequence features, the sequence-processing components of DMIN, DIEN, and DMR are vacuous, so we use DeepFM as the base model and report only DeepFM here.
LoopFM-only achieves +1.6\% AUC gain, substantially larger than the +0.6\% on the baseline with sequences (Table~\ref{tab:public_kuai}), confirming that LoopFM's value is greatest when the VM lacks its own sequence feature inputs.
KD+LoopFM (0.7146) slightly underperforms LoopFM-only (0.7167) here, likely because KD parameters are not re-tuned on this no-sequence baseline.
Notably, DeepFM's no-sequence baseline (0.7056) is even higher than its baseline with sequences (0.6877 in Table~\ref{tab:public_kuai}), as discussed above.

\section{Amazon Electronics Results}
\label{app:amazon}

Amazon Electronics~\citep{he2016ups} contains user reviews on electronics products with features: user\_id, item\_id, cate\_id, and pre-built behavioral sequences (item\_history, cate\_history).
Unlike TaobaoAd, this dataset does not contain timestamps; we use the provided sequential ordering and split into 8 temporal chunks following the same streaming protocol.
The FM and VMs use BARS-tuned hyperparameters (emb\_dim=64, hidden=[1024, 512, 256]).

\begin{table}[h]
\caption{LoopFM on Amazon Electronics (FM: DMIN) across four DNN VM architectures. KD: $\alpha{=}5$, $\gamma{=}10$, $\beta{=}5$.}
\label{tab:public_amazon}
\centering
\small
\begin{tabular}{lcccccccc}
\toprule
& \multicolumn{2}{c}{Deep\underline{FM}} & \multicolumn{2}{c}{DIEN} & \multicolumn{2}{c}{DMR} & \multicolumn{2}{c}{DMIN} \\
\cmidrule(lr){2-3} \cmidrule(lr){4-5} \cmidrule(lr){6-7} \cmidrule(lr){8-9}
Method & AUC & LogLoss & AUC & LogLoss & AUC & LogLoss & AUC & LogLoss \\
\midrule
w/o distill & 0.8035 & 0.6327 & 0.8490 & 0.4948 & 0.8456 & 0.4882 & 0.8386 & 0.4947 \\
KD & 0.8119 & 0.6086 & 0.8640 & \textbf{0.4627} & 0.8631 & 0.4622 & 0.8551 & \textbf{0.4751} \\
LoopFM-only & 0.8127 & 0.6110 & 0.8537 & 0.4772 & 0.8475 & 0.4866 & 0.8388 & 0.4995 \\
KD + LoopFM & \textbf{0.8179} & \textbf{0.5710} & \textbf{0.8645} & 0.4646 & \textbf{0.8650} & \textbf{0.4608} & \textbf{0.8554} & 0.4809 \\
\bottomrule
\end{tabular}
\end{table}

Table~\ref{tab:public_amazon} presents results on Amazon Electronics across four DNN-based VMs.
Similar to KuaiVideo, we exclude \underline{FM} and Fm\underline{FM} because Amazon's dominant features are behavioral sequences (item\_history, cate\_history), which these non-sequential models cannot process well.
KD + LoopFM achieves the best AUC for all four VMs: Deep\underline{FM} (0.8179), DIEN (0.8645), DMR (0.8650), and DMIN (0.8554).
Note that while AUC consistently improves, LogLoss slightly degrades for DIEN and DMIN when adding LoopFM to KD, suggesting that the additional sequence features can affect calibration even when ranking quality improves.

Compared to TaobaoAd, KD is the dominant signal on Amazon (+1.8--2.1\% AUC for sequential VMs), while LoopFM-only provides modest gains (+0.02--1.1\%).
Similar to KuaiVideo, we attribute this to the limited feature richness---and even more so here: Amazon has only 6 features (user\_id, item\_id, cate\_id, and their history sequences), so the existing sequences already capture most of the available signal, leaving minimal cross-feature interactions for LoopFM to add values.
KD + LoopFM nonetheless consistently outperforms KD alone for all four VMs, with DMR achieving the best overall AUC of 0.8650 (+2.29\% relative over baseline).

\section{Seed Variance Study}
\label{app:variance}

To verify statistical significance of single-run results, we repeat the four main configurations (Baseline, KD, LoopFM-only, KD+LoopFM) on TaobaoAd with 5 random seeds using Deep\underline{FM} as VM with the same setup as Table~\ref{tab:public_taobao} (no shuffle, streaming, 1 epoch).
Table~\ref{tab:pipeline_variance} reports the results.
All gains are highly statistically significant under paired $t$-tests ($p < 0.001$): LoopFM-only achieves $+$0.0347 $\pm$ 0.0004 AUC over baseline ($t{=}189.5$), and KD+LoopFM achieves $+$0.0458 $\pm$ 0.0004 ($t{=}265.2$).
KD+LoopFM has the lowest variance (std $= 0.0001$), suggesting the two transfer channels stabilize each other.

\begin{table}[h]
\centering
\small
\caption{Full pipeline variance study on TaobaoAd (VM: Deep\underline{FM}, FM: DMIN, 5 seeds). All gains over baseline are significant at $p < 0.001$ (paired $t$-test, $df{=}4$).}
\label{tab:pipeline_variance}
\begin{tabular}{lccc}
\toprule
Method & AUC (mean $\pm$ std) & Range & $\Delta$AUC vs.\ baseline \\
\midrule
Baseline & $0.5884 \pm 0.0004$ & 0.0011 & --- \\
KD & $0.5978 \pm 0.0004$ & 0.0011 & $+$0.0094 ($t{=}32.6$) \\
LoopFM-only & $0.6230 \pm 0.0005$ & 0.0015 & $+$0.0347 ($t{=}189.5$) \\
KD + LoopFM & $\mathbf{0.6342 \pm 0.0001}$ & 0.0003 & $+$0.0458 ($t{=}265.2$) \\
\bottomrule
\end{tabular}
\end{table}


\section{Formal Assumptions for Theoretical Analysis}
\label{app:assumptions}

We state the assumptions used in Theorem~\ref{thm:loopfm-tr}. All notation follows Sections~\ref{sec:method} and~\ref{sec:theory}.

\begin{itemize}[leftmargin=2em, itemsep=4pt]
    \item[\textbf{(A1)}] \textbf{NTK-linearized benign-overfitting assumptions.}
    For each FM$_k$, define the expected excess risk $\epsilon_{\text{over}}(p_k,m_k) := R_{\text{ach}}(\text{FM}_k) - \mathcal{R}^{*(\text{FM}_k)}$ with $n$ training samples.
    Let $\phi_k(u,a)\in\mathbb{R}^{p_k}$ denote the local linearized feature map and
    $\Sigma_k:=\mathbb{E}[\phi_k\phi_k^\top]$ with eigenvalues
    $\lambda_{k,1}\ge\cdots\ge\lambda_{k,p_k}\ge0$. Here $m_k$ denotes the Bartlett split index, with $1\le m_k<p_k$. We assume:
    \begin{enumerate}[label=(A1.\roman*), leftmargin=1.7em, itemsep=1pt]
        \item \emph{NTK/lazy-training linearization:} training remains in a local kernel regime so FM$_k$ is well-approximated by a linear predictor on $\phi_k$~\citep{jacot2018neural}.
        \item \emph{Bartlett benign-overfitting conditions:} for split index $m_k<p_k$, the tail functionals $r_{m_k}(\Sigma_k)$ and $R_{m_k}(\Sigma_k)$ satisfy the finite-sample conditions in \citet[Theorem~1]{bartlett2020benign}, with $p_k\gg n$ and $m_k=o(n)$.
    \end{enumerate}
    Then Appendix~\ref{app:overparameterization} yields the two-sided envelope:
    \begin{equation}\label{eq:a1-envelope}
    \underline C_{\text{over}}\,\sigma^2\,\xi_k
    \le
    \epsilon_{\text{over}}(p_k,m_k)
    \le
    \bar C_{\text{over}}\,\sigma^2\,\xi_k,
    \end{equation}
    where $\xi_k:=m_k/n + n/(p_k-m_k) \to 0$ as $p_k,n\to\infty$, and $0 < \underline C_{\text{over}}\le \bar C_{\text{over}}$ are shared constants.

    \paragraph{Justification for (A1).}
    The specialization $\xi_k = m_k/n + n/(p_k - m_k)$ assumes a bi-level (spiked) covariance where the $p_k - m_k$ tail eigenvalues are approximately equal, giving tail effective rank $R_{m_k} = p_k - m_k$ (see Corollary~\ref{cor:a1-from-linearized} for the derivation).
    This is natural for overparameterized FMs: the NTK covariance has $m_k$ signal directions aligned with input features, and $p_k - m_k \gg n$ near-uniform tail directions from overparameterization, matching the bi-level structure of \citet[Theorem~2.2]{bartlett2020benign}.
    For more general covariance structures, Theorem~\ref{thm:loopfm-tr-general} (Appendix~\ref{app:general-tr-bound}) gives the same TR bound with $\xi_k = m_k/n + n/R_{m_k}(\Sigma_k)$, valid under any benign covariance satisfying \citet[Definition~4]{bartlett2020benign}.

    \item[\textbf{(A2)}] \textbf{Per-feature mutual information bound.}
    Write the additional FM$_2$ feature block as $\mathbf{x}_{\text{extra},2} = (u_{m_1+1},\ldots,u_{m_2})$.
    For each $j\in\{m_1{+}1,\ldots,m_2\}$:
    \begin{enumerate}[label=(\alph*)]
    \item \emph{Current-step:} there exist constants $0<\underline{\kappa}_{\mathrm{gap}}\le\bar{\kappa}_{\mathrm{gap}}$ such that
    \begin{equation}\label{eq:a2-kappa-gap}
    \underline{\kappa}_{\mathrm{gap}}
    \le
    I\!\left(u_j^{(t)}; y \mid \mathbf{x}_{\text{VM}}^{(t)},\mathbf{x}_{\text{extra},1}^{(t)},u_{m_1+1:j-1}^{(t)}\right)
    \le \bar{\kappa}_{\mathrm{gap}}.
    \end{equation}
    \item \emph{Historical:} there exist constants $0<\underline{\kappa}_{\mathrm{gap}}^{\mathrm{hist}}\le\bar{\kappa}_{\mathrm{gap}}^{\mathrm{hist}}$ such that
    \begin{equation}\label{eq:a2-kappa-gap-hist}
    \underline{\kappa}_{\mathrm{gap}}^{\mathrm{hist}}
    \le
    I\!\left(u_j^{(t_1:t_L)}; y \mid \mathbf{x}_{\text{VM}}^{(t)},\mathbf{H}_u,\mathbf{x}_{\text{extra},1}^{(t_1:t_L)},u_{m_1+1:j-1}^{(t_1:t_L)}\right)
    \le \bar{\kappa}_{\mathrm{gap}}^{\mathrm{hist}}.
    \end{equation}
    \end{enumerate}

    \paragraph{Justification for (A2).}
    (A2) states that each additional feature group in FM$_2$ contributes bounded conditional MI to the label.
    The lower bound $\underline{\kappa}_{\mathrm{gap}} > 0$ encodes that each new feature carries non-trivial predictive signal---this is the premise of adding features to FM$_2$.
    The upper bound $\bar{\kappa}_{\mathrm{gap}}$ reflects diminishing returns: conditioned on prior features, each additional feature's marginal contribution is bounded (a consequence of the label having finite entropy $H(y) \le \log 2$).
    The historical variant (b) bounds the same quantity over the full event history, with potentially different constants since $L$ timesteps of a feature carry more information than a single timestep.

    \item[\textbf{(A3)}] \textbf{Non-worsening cross-platform pipeline quality.}
    FM$_2$ uses at least as many parameters ($p_2\ge p_1$), the same or larger autoencoder bottleneck ($d_2\ge d_1$), and the same or finer quantization ($b_2\ge b_1$). We require:
    \begin{equation}\label{eq:a3-aeq}
    \ell_{\mathrm{repr},2}^{\mathrm{cross}} + \ell_{\mathrm{AE},2}^{\mathrm{cross}} + \ell_{\mathrm{Q},2}^{\mathrm{cross}}
    \;\le\;
    \ell_{\mathrm{repr},1}^{\mathrm{cross}} + \ell_{\mathrm{AE},1}^{\mathrm{cross}} + \ell_{\mathrm{Q},1}^{\mathrm{cross}},
    \end{equation}
    where $\ell_{\mathrm{repr},k}^{\mathrm{cross}}$, $\ell_{\mathrm{AE},k}^{\mathrm{cross}}$, and $\ell_{\mathrm{Q},k}^{\mathrm{cross}}$ are the cross-platform representation, autoencoder, and quantization losses defined in~\eqref{eq:ell-repr-cross}--\eqref{eq:ell-q-cross}.
\end{itemize}

\paragraph{Justification for including $\ell_{\mathrm{repr},k}^{\mathrm{cross}}$ in (A3).}
The AE and quantization components are directly controlled by design ($d_2\ge d_1$, $b_2\ge b_1$).
The representation loss $\ell_{\mathrm{repr},k}^{\mathrm{cross}}$ measures how much cross-platform information the FM's embedding $\mathbf{E}_u^{(k)}$ fails to capture from the raw extra features $\mathbf{x}_{\text{extra},k}^{(t_1:t_L)}$.
A larger FM ($p_2\ge p_1$) with lower excess risk produces richer intermediate representations at each time step, reducing the information gap between raw features and embeddings.
Therefore, the total cross-platform pipeline loss (representation + AE + quantization) for FM$_2$ should not exceed that of FM$_1$ when $p_2\ge p_1$, $d_2\ge d_1$, and $b_2\ge b_1$, making~(A3) a natural and mild assumption.

\section{Proofs for Theoretical Analysis}
\label{app:theory-proofs}

\subsection{Formal Statement and Proof of Theorem~\ref{thm:gain-decomp} (Gain Decomposition)}
\label{app:proof-gain-decomp}

We first restate Theorem~\ref{thm:gain-decomp} with full detail, including the \emph{Temporally-Privileged VM} (Temp-Priv) configuration used in the formal definitions.

\paragraph{Additional configuration.} \textbf{Temp-Priv VM}: a hypothetical VM with access to both current features $\mathbf{x}_{\text{VM}}^{(t)}$ and the full raw VM-side history $\mathbf{H}_u = (\mathbf{x}_{\text{VM}}^{(t_1)}, \ldots, \mathbf{x}_{\text{VM}}^{(t_L)})$. This serves as a theoretical upper bound for Self-LoopFM, since raw features contain strictly more information than any compressed embedding derived from them (by the data processing inequality).

\begin{theorem}[Gain decomposition; formal restatement of Theorem~\ref{thm:gain-decomp}]\label{thm:gain-decomp-formal}
Let $\mathcal{R}^*(\cdot) = H(y \mid \cdot)$ denote the Bayes risk under binary cross-entropy. The LoopFM information gain $\mathcal{I}_{\text{LoopFM}}(\text{FM}_k) := \mathcal{R}^{*(\text{VM})} - \mathcal{R}^{*(\text{LoopFM}_k)}$ decomposes exactly as:
\begin{equation}
    \mathcal{I}_{\text{LoopFM}}(\text{FM}_k)
    \;=\; \mathcal{I}_{\mathrm{temporal}}
    \;+\; \mathcal{I}_{\mathrm{cross},k}
    \;-\; \mathcal{I}_{\mathrm{residual},k},
\end{equation}
where all three terms are non-negative mutual informations:
\begin{enumerate}[label=(\roman*), itemsep=3pt, topsep=3pt]
    \item \textbf{Temporal information:} $\mathcal{I}_{\mathrm{temporal}} := I(\mathbf{H}_u; y \mid \mathbf{x}_{\text{VM}}^{(t)}) = \mathcal{R}^{*(\text{VM})} - \mathcal{R}^{*(\text{Temp-Priv})} \ge 0$. This is the predictive information in the user's raw feature history beyond what the current features reveal. It is a property of the data distribution alone, independent of any FM.
    \item \textbf{Cross-feature information:} $\mathcal{I}_{\mathrm{cross},k} := I(\mathbf{S}_u^{(k)}; y \mid \mathbf{x}_{\text{VM}}^{(t)}, \mathbf{H}_u) \ge 0$. This is the additional predictive information that the FM's extra features $\mathbf{x}_{\text{extra}}$ contribute, as encoded in the LoopFM embeddings, beyond what the raw VM history already provides. For Self-LoopFM (where FM$_k$ = VM), this term equals zero since the VM's own embeddings cannot introduce features the VM didn't have.
    \item \textbf{Compression residual:} $\mathcal{I}_{\mathrm{residual},k} := I(\mathbf{H}_u; y \mid \mathbf{x}_{\text{VM}}^{(t)}, \mathbf{S}_u^{(k)}) \ge 0$. This is the predictive information in $\mathbf{H}_u$ that is \emph{not} captured by the LoopFM sequence $\mathbf{S}_u^{(k)}$, representing the cost of the compression pipeline (FM representation $\to$ autoencoder $\to$ quantization).
\end{enumerate}
The following special cases and bounds hold:
\begin{enumerate}[label=(\alph*), itemsep=2pt, topsep=3pt]
    \item \textbf{Self-LoopFM:} $\mathcal{I}_{\text{LoopFM}}(\text{Self}) = \mathcal{I}_{\mathrm{temporal}} - \mathcal{I}_{\mathrm{residual,self}}$. LoopFM dominates Self-LoopFM whenever $\mathcal{I}_{\mathrm{residual},k} \le \mathcal{I}_{\mathrm{residual,self}}$.
    \item \textbf{Temporal upper bound:} $\mathcal{I}_{\text{LoopFM}}(\text{Self}) \le \mathcal{I}_{\mathrm{temporal}}$, with equality iff $\mathbf{S}_u^{\text{self}}$ is a sufficient statistic for $\mathbf{H}_u$ w.r.t.\ $y$ given $\mathbf{x}_{\text{VM}}^{(t)}$.
    \item \textbf{Cross-feature upper bound:} By the data processing inequality, $\mathcal{I}_{\mathrm{cross},k} \le I(\mathbf{x}_{\text{extra},k}^{(t_1:t_L)}; y \mid \mathbf{x}_{\text{VM}}^{(t)}, \mathbf{H}_u) =: \mathcal{I}_{\mathrm{feature\text{-}raw},k}$.
\end{enumerate}
\end{theorem}

\begin{proof}
The proof applies the chain rule of mutual information to the pair $(\mathbf{S}_u^{(k)}, \mathbf{H}_u)$ in two orderings, where $\mathbf{H}_u := (\mathbf{x}_{\text{VM}}^{(t_1)}, \ldots, \mathbf{x}_{\text{VM}}^{(t_L)})$ is the raw VM feature history.

\textbf{Step 1: Chain rule in two orderings.}
Applying the chain rule to the joint mutual information $I(\mathbf{S}_u^{(k)}, \mathbf{H}_u; y \mid \mathbf{x}_{\text{VM}}^{(t)})$:

\emph{Ordering 1} ($\mathbf{H}_u$ first):
\begin{equation}\label{eq:chain-ordering-1}
    I(\mathbf{S}_u^{(k)}, \mathbf{H}_u;\; y \mid \mathbf{x}_{\text{VM}}^{(t)})
    = \underbrace{I(\mathbf{H}_u;\; y \mid \mathbf{x}_{\text{VM}}^{(t)})}_{\mathcal{I}_{\mathrm{temporal}}}
    + \underbrace{I(\mathbf{S}_u^{(k)};\; y \mid \mathbf{x}_{\text{VM}}^{(t)}, \mathbf{H}_u)}_{\mathcal{I}_{\mathrm{cross},k}}.
\end{equation}
Conditioning on $\mathbf{H}_u$ first captures the temporal information; any remaining information in $\mathbf{S}_u^{(k)}$ about $y$ given $(\mathbf{x}_{\text{VM}}^{(t)}, \mathbf{H}_u)$ is the cross-platform contribution.

\emph{Ordering 2} ($\mathbf{S}_u^{(k)}$ first):
\begin{equation}\label{eq:chain-ordering-2}
    I(\mathbf{S}_u^{(k)}, \mathbf{H}_u;\; y \mid \mathbf{x}_{\text{VM}}^{(t)})
    = \underbrace{I(\mathbf{S}_u^{(k)};\; y \mid \mathbf{x}_{\text{VM}}^{(t)})}_{\mathcal{I}_{\text{LoopFM}}(\text{FM}_k)}
    + \underbrace{I(\mathbf{H}_u;\; y \mid \mathbf{x}_{\text{VM}}^{(t)}, \mathbf{S}_u^{(k)})}_{\mathcal{I}_{\mathrm{residual},k}}.
\end{equation}

\textbf{Step 2: Equate and rearrange.}
Since both orderings compute the same joint mutual information, equating~\eqref{eq:chain-ordering-1} and~\eqref{eq:chain-ordering-2}:
\begin{equation}
    \mathcal{I}_{\text{LoopFM}}(\text{FM}_k) + \mathcal{I}_{\mathrm{residual},k}
    = \mathcal{I}_{\mathrm{temporal}} + \mathcal{I}_{\mathrm{cross},k}.
\end{equation}
Rearranging:
\begin{equation}
    \mathcal{I}_{\text{LoopFM}}(\text{FM}_k)
    = \mathcal{I}_{\mathrm{temporal}} + \mathcal{I}_{\mathrm{cross},k} - \mathcal{I}_{\mathrm{residual},k}.
\end{equation}
Non-negativity of all three terms follows from the non-negativity of mutual information.
The identification $\mathcal{I}_{\mathrm{temporal}} = \mathcal{R}^{*(\text{VM})} - \mathcal{R}^{*(\text{Temp-Priv})}$ follows from the conditional MI identity $I(X;Y\mid Z) = H(Y\mid Z) - H(Y\mid X,Z)$ and the BCE Bayes-risk identity $\mathcal{R}^*(\cdot)=H(y\mid\cdot)$.

\textbf{Step 3: Self-LoopFM special case (part~(i)).}
When FM$_k$ = VM, the embedding $\mathbf{z}_k^{(t_j)} = \psi(\mathbf{x}_{\text{VM}}^{(t_j)})$ is a deterministic function of $\mathbf{x}_{\text{VM}}^{(t_j)}$.
Therefore $\mathbf{S}_u^{\text{self}} = [\psi(\mathbf{x}_{\text{VM}}^{(t_1)}), \ldots, \psi(\mathbf{x}_{\text{VM}}^{(t_L)})]$ is a function of $\mathbf{H}_u$.
By the data processing inequality, $I(\mathbf{S}_u^{\text{self}}; y \mid \mathbf{x}_{\text{VM}}^{(t)}, \mathbf{H}_u) = 0$ since $\mathbf{S}_u^{\text{self}}$ is $\sigma(\mathbf{H}_u)$-measurable.
Hence $\mathcal{I}_{\mathrm{cross,self}} = 0$, giving $\mathcal{I}_{\text{LoopFM}}(\text{Self}) = \mathcal{I}_{\mathrm{temporal}} - \mathcal{I}_{\mathrm{residual,self}}$.
Note that LoopFM similarly accesses temporal information through its FM-derived sequence $\mathbf{S}_u^{(k)}$; the only difference is the source of event embeddings (FM vs.\ VM).
If the retention-dominance condition $\mathcal{I}_{\mathrm{residual},k} \le \mathcal{I}_{\mathrm{residual,self}}$ holds, LoopFM is no worse than Self-LoopFM.

\textbf{Step 4: Temporal upper bound (part~(ii)).}
Since $\mathcal{I}_{\mathrm{residual,self}} \ge 0$, we have $\mathcal{I}_{\text{LoopFM}}(\text{Self}) \le \mathcal{I}_{\mathrm{temporal}}$.
Equality holds when $\mathcal{I}_{\mathrm{residual,self}} = I(\mathbf{H}_u; y \mid \mathbf{x}_{\text{VM}}^{(t)}, \mathbf{S}_u^{\text{self}}) = 0$, i.e., when $\mathbf{S}_u^{\text{self}}$ is a sufficient statistic for $\mathbf{H}_u$ with respect to $y$ given $\mathbf{x}_{\text{VM}}^{(t)}$.

\textbf{Step 5: Cross-platform upper bound (part~(iii)).}
The FM$_k$ embedding at time $t_j$ is computed as $\mathbf{z}_{k,t_j} = f_{\mathrm{enc}}(f_{\text{FM}_k}(\mathbf{x}_{\text{FM}_k}^{(t_j)}))$, where $\mathbf{x}_{\text{FM}_k}^{(t_j)} = (\mathbf{x}_{\text{VM}}^{(t_j)}, \mathbf{x}_{\text{extra},k}^{(t_j)})$.
Since both $f_{\text{FM}_k}$ and $f_{\mathrm{enc}}$ are deterministic functions, $\mathbf{S}_u^{(k)}$ is a deterministic function of $(\mathbf{H}_u, \mathbf{x}_{\text{extra},k}^{(t_1:t_L)})$.
Conditioning on $\mathbf{H}_u$ makes $\mathbf{S}_u^{(k)}$ a function of $\mathbf{x}_{\text{extra},k}^{(t_1:t_L)}$ alone.
By the data processing inequality applied to the Markov chain $\mathbf{x}_{\text{extra},k}^{(t_1:t_L)} \to \mathbf{S}_u^{(k)} \to y$ (conditioned on $(\mathbf{x}_{\text{VM}}^{(t)}, \mathbf{H}_u)$):
\begin{equation}
    \mathcal{I}_{\mathrm{cross},k} = I(\mathbf{S}_u^{(k)}; y \mid \mathbf{x}_{\text{VM}}^{(t)}, \mathbf{H}_u)
    \le I(\mathbf{x}_{\text{extra},k}^{(t_1:t_L)}; y \mid \mathbf{x}_{\text{VM}}^{(t)}, \mathbf{H}_u)
    = \mathcal{I}_{\mathrm{feature\text{-}raw},k}. \qedhere
\end{equation}
\end{proof}

\subsection{Pipeline Decomposition of Compression Loss}
\label{app:pipeline-decomp}

\begin{proposition}[Pipeline decomposition]\label{prop:pipeline-decomp}
The residual $\mathcal{I}_{\mathrm{residual},k}$ decomposes along the LoopFM pipeline:
raw FM embeddings $\mathbf{E}_u^{(k)} \in \mathbb{R}^{L \times D}$ ($L$ timesteps, $D$-dim per step) $\to$ autoencoder-compressed $\mathbf{Z}_u^{(k)} \in \mathbb{R}^{L \times d}$ ($d \ll D$) $\to$ quantized $\mathbf{S}_u^{(k)}$ ($b_k$ bits).
Since $\sigma(\mathbf{S}_u^{(k)}) \subseteq \sigma(\mathbf{Z}_u^{(k)}) \subseteq \sigma(\mathbf{E}_u^{(k)})$:
\begin{equation}\label{eq:pipeline-bound}
    \mathcal{I}_{\mathrm{residual},k} \;\le\; \underbrace{\ell_{\mathrm{repr},k}(p_k)}_{\substack{\text{FM repr.} \\ \text{residual}}} \;+\; \underbrace{\ell_{\mathrm{AE},k}(d)}_{\substack{\text{AE compr.} \\ \text{loss}}} \;+\; \underbrace{\ell_{\mathrm{Q},k}(b_k)}_{\substack{\text{quantization} \\ \text{loss}}},
\end{equation}
where:
\begin{align}
    \ell_{\mathrm{repr},k}(p_k) &:= I\!\left(\mathbf{H}_u;\; y \;\middle|\; \mathbf{x}_{\text{VM}}^{(t)},\, \mathbf{E}_u^{(k)}\right) \ge 0, \label{eq:ell-repr} \\
    \ell_{\mathrm{AE},k}(d) &:= I\!\left(\mathbf{H}_u;\; \mathbf{E}_u^{(k)} \;\middle|\; \mathbf{x}_{\text{VM}}^{(t)}\right) - I\!\left(\mathbf{H}_u;\; \mathbf{Z}_u^{(k)} \;\middle|\; \mathbf{x}_{\text{VM}}^{(t)}\right) \ge 0, \label{eq:ell-ae} \\
    \ell_{\mathrm{Q},k}(b_k) &:= I\!\left(\mathbf{H}_u;\; \mathbf{Z}_u^{(k)} \;\middle|\; \mathbf{x}_{\text{VM}}^{(t)}\right) - I\!\left(\mathbf{H}_u;\; \mathbf{S}_u^{(k)} \;\middle|\; \mathbf{x}_{\text{VM}}^{(t)}\right) \ge 0. \label{eq:ell-q}
\end{align}
The three terms correspond to successive stages of the pipeline:
$\ell_{\mathrm{repr},k}$ is the \emph{representation residual}---the predictive information in user history $\mathbf{H}_u$ about $y$ that the FM's raw embeddings $\mathbf{E}_u^{(k)}$ fail to retain. A larger FM (more parameters $p_k$) produces richer intermediate representations that retain more of this information, driving $\ell_{\mathrm{repr},k}\to 0$.
$\ell_{\mathrm{AE},k}$ is the \emph{autoencoder compression loss}---the information about $\mathbf{H}_u$ lost when compressing $\mathbf{E}_u^{(k)}\in\mathbb{R}^{L\times D}$ to $\mathbf{Z}_u^{(k)}\in\mathbb{R}^{L\times d}$; decreases with larger bottleneck dimension $d$.
$\ell_{\mathrm{Q},k}$ is the \emph{quantization loss}---the information lost when discretizing $\mathbf{Z}_u^{(k)}$ to $b_k$ bits; decreases with finer quantization.
\end{proposition}

\begin{proof}
Let
\[
X:=\mathbf{x}_{\text{VM}}^{(t)},\quad
H:=\mathbf{H}_u,\quad
E_p:=\mathbf{E}_u^{(k)}(p_k),\quad
Z_{p,d}:=\mathbf{Z}_u^{(k)}(p_k,d),\quad
S_{p,d,b}:=\mathbf{S}_u^{(k)}(p_k,d,b_k).
\]

\textbf{Step 1: Decompose the residual into representation and downstream losses.}
Starting from
\[
\mathcal{I}_{\mathrm{residual},k}=I(H;y\mid X,S_{p,d,b}),
\]
add and subtract $I(H;y\mid X,E_p)$:
\begin{align}
\mathcal{I}_{\mathrm{residual},k}
&= \underbrace{I(H;y\mid X,E_p)}_{\ell_{\mathrm{repr},k}(p_k)}
 + \bigl[I(H;y\mid X,S_{p,d,b})-I(H;y\mid X,E_p)\bigr]. \label{eq:residual-split-proof}
\end{align}
Using the chain-rule identity
\[
I(A;B\mid C)-I(A;B\mid C,D)=I(A;D\mid C)-I(A;D\mid B,C),
\]
with $(A,B,C,D)=(H,y,(X,S_{p,d,b}),E_p)$, we get
\begin{align}
I(H;y\mid X,S_{p,d,b})-I(H;y\mid X,E_p)
&= I(H;E_p\mid X,S_{p,d,b})
 - I(H;E_p\mid X,S_{p,d,b},y) \notag\\
&\le I(H;E_p\mid X,S_{p,d,b}). \label{eq:additional-loss-bound-proof}
\end{align}

\textbf{Step 2: Telescope through AE and quantization.}
Since $\sigma(S_{p,d,b})\subseteq\sigma(Z_{p,d})\subseteq\sigma(E_p)$, the chain rule gives
\begin{align}
I(H;E_p\mid X,S_{p,d,b})
&= I(H;E_p\mid X)-I(H;S_{p,d,b}\mid X) \notag\\
&= \bigl[I(H;E_p\mid X)-I(H;Z_{p,d}\mid X)\bigr] \notag\\
&\quad + \bigl[I(H;Z_{p,d}\mid X)-I(H;S_{p,d,b}\mid X)\bigr] \notag\\
&= \ell_{\mathrm{AE},k}(d)+\ell_{\mathrm{Q},k}(b_k). \label{eq:telescope-ae-q-proof}
\end{align}
Combining \eqref{eq:residual-split-proof}, \eqref{eq:additional-loss-bound-proof}, and \eqref{eq:telescope-ae-q-proof} yields
\[
\mathcal{I}_{\mathrm{residual},k}
\le \ell_{\mathrm{repr},k}(p_k)+\ell_{\mathrm{AE},k}(d)+\ell_{\mathrm{Q},k}(b_k),
\]
which is Eq.~\eqref{eq:pipeline-bound}.

\textbf{Step 3: Non-negativity of all three losses.}
\begin{itemize}[leftmargin=1.2em, itemsep=1pt]
    \item $\ell_{\mathrm{repr},k}(p_k)=I(H;y\mid X,E_p)\ge0$ by non-negativity of conditional MI.
    \item $\ell_{\mathrm{AE},k}(d)=I(H;E_p\mid X)-I(H;Z_{p,d}\mid X)\ge0$ by DPI under $H\to E_p\to Z_{p,d}$ (given $X$).
    \item $\ell_{\mathrm{Q},k}(b_k)=I(H;Z_{p,d}\mid X)-I(H;S_{p,d,b}\mid X)\ge0$ by DPI under $H\to Z_{p,d}\to S_{p,d,b}$ (given $X$).
\end{itemize}

\textbf{Step 4: Connection between $\ell_{\mathrm{AE},k}(d)$ and AE MSE.}
Let $\hat E_{p,d}:=f_{\mathrm{dec}}(Z_{p,d})$ and
\[
D_{\mathrm{AE},k}(d):=\frac{1}{LD}\,\mathbb{E}\!\left[\|E_p-\hat E_{p,d}\|_2^2\right].
\]
Because $\hat E_{p,d}$ is a deterministic function of $Z_{p,d}$, DPI gives
\begin{equation}
\ell_{\mathrm{AE},k}(d)
= I(H;E_p\mid X)-I(H;Z_{p,d}\mid X)
\le I(H;E_p\mid X)-I(H;\hat E_{p,d}\mid X). \label{eq:ae-dpi-to-recon-proof}
\end{equation}
Now impose the linear-Gaussian approximation conditioned on $X$:
\[
E_p = A_k H + \varepsilon_{\mathrm{repr}},\qquad
\hat E_{p,d} = E_p + \varepsilon_{\mathrm{AE}},
\]
with $\varepsilon_{\mathrm{repr}}\sim\mathcal{N}(0,\Sigma_{\mathrm{repr},k})$, $\varepsilon_{\mathrm{AE}}\sim\mathcal{N}(0,q_k(d)I)$, independent of each other and of $H$, and $q_k(d)=D_{\mathrm{AE},k}(d)$. If $\Sigma_{\mathrm{signal},k}:=A_k\Sigma_{H\mid X}A_k^\top$ and $\Sigma_{\mathrm{repr},k}$ are simultaneously diagonalizable with eigenvalues $(\lambda_{k,j},\rho_{k,j})_{j=1}^{r_k}$, then
\begin{align}
I(H;E_p\mid X)-I(H;\hat E_{p,d}\mid X)
&= \frac12\sum_{j=1}^{r_k}
\log\frac{1+\lambda_{k,j}/\rho_{k,j}}{1+\lambda_{k,j}/(\rho_{k,j}+q_k(d))} \notag\\
&\le \frac12\sum_{j=1}^{r_k}\log\left(1+\frac{q_k(d)}{\rho_{k,j}}\right) \notag\\
&\le \frac{r_k}{2}\log\left(1+\frac{q_k(d)}{\sigma_{\mathrm{repr},k,\min}^2}\right), \label{eq:ae-mse-bound-proof}
\end{align}
where $\sigma_{\mathrm{repr},k,\min}^2:=\min_j\rho_{k,j}$. Combining \eqref{eq:ae-dpi-to-recon-proof} and \eqref{eq:ae-mse-bound-proof} with $q_k(d)=D_{\mathrm{AE},k}(d)$ gives the claimed AE-MSE control for $\ell_{\mathrm{AE},k}(d)$.
\end{proof}

\paragraph{Cross-platform pipeline losses.}
Conditioned on $(\mathbf{x}_{\text{VM}}^{(t)}, \mathbf{H}_u)$, the Markov chain $\mathbf{x}_{\text{extra},k}^{(t_1:t_L)} \to \mathbf{E}_u^{(k)} \to \mathbf{Z}_u^{(k)} \to \mathbf{S}_u^{(k)}$ yields cross-platform pipeline losses:
\begin{align}
    \ell_{\mathrm{repr},k}^{\mathrm{cross}} &:= I\!\left(\mathbf{x}_{\text{extra},k}^{(t_1:t_L)};\; y \;\middle|\; \mathbf{x}_{\text{VM}}^{(t)},\, \mathbf{H}_u\right) - I\!\left(\mathbf{E}_u^{(k)};\; y \;\middle|\; \mathbf{x}_{\text{VM}}^{(t)},\, \mathbf{H}_u\right) \ge 0, \label{eq:ell-repr-cross} \\
    \ell_{\mathrm{AE},k}^{\mathrm{cross}} &:= I\!\left(\mathbf{E}_u^{(k)};\; y \;\middle|\; \mathbf{x}_{\text{VM}}^{(t)},\, \mathbf{H}_u\right) - I\!\left(\mathbf{Z}_u^{(k)};\; y \;\middle|\; \mathbf{x}_{\text{VM}}^{(t)},\, \mathbf{H}_u\right) \ge 0, \label{eq:ell-ae-cross} \\
    \ell_{\mathrm{Q},k}^{\mathrm{cross}} &:= I\!\left(\mathbf{Z}_u^{(k)};\; y \;\middle|\; \mathbf{x}_{\text{VM}}^{(t)},\, \mathbf{H}_u\right) - I\!\left(\mathbf{S}_u^{(k)};\; y \;\middle|\; \mathbf{x}_{\text{VM}}^{(t)},\, \mathbf{H}_u\right) \ge 0. \label{eq:ell-q-cross}
\end{align}
Unlike the temporal losses (which measure information about $\mathbf{H}_u$), the cross-platform losses measure information about $y$ from the FM-VM feature gap, conditioned on $(\mathbf{x}_{\text{VM}}^{(t)}, \mathbf{H}_u)$:
$\ell_{\mathrm{repr},k}^{\mathrm{cross}}$ is the predictive signal from the raw extra features $\mathbf{x}_{\text{extra},k}^{(t_1:t_L)}$ that the FM's embeddings $\mathbf{E}_u^{(k)}$ fail to retain---the FM-VM feature gap information lost during the FM's forward pass;
$\ell_{\mathrm{AE},k}^{\mathrm{cross}}$ is the feature-gap information lost during autoencoder compression ($\mathbf{E}\to\mathbf{Z}$);
$\ell_{\mathrm{Q},k}^{\mathrm{cross}}$ is the feature-gap information lost during quantization ($\mathbf{Z}\to\mathbf{S}$).
Their sum determines $\eta_k$: the fraction of raw cross-platform information destroyed by the pipeline.

\begin{lemma}[Cross-platform retention]\label{lem:eta-pipeline}
$\mathcal{I}_{\mathrm{feature\text{-}raw},k} - \mathcal{I}_{\mathrm{cross},k} = \ell_{\mathrm{repr},k}^{\mathrm{cross}} + \ell_{\mathrm{AE},k}^{\mathrm{cross}} + \ell_{\mathrm{Q},k}^{\mathrm{cross}}$.
Equivalently, $\mathcal{I}_{\mathrm{cross},k} = (1 - \eta_k) \cdot \mathcal{I}_{\mathrm{feature\text{-}raw},k}$ where
$\eta_k := (\ell_{\mathrm{repr},k}^{\mathrm{cross}} + \ell_{\mathrm{AE},k}^{\mathrm{cross}} + \ell_{\mathrm{Q},k}^{\mathrm{cross}}) / \mathcal{I}_{\mathrm{feature\text{-}raw},k}$.
\end{lemma}

\begin{proof}
Conditioned on $(\mathbf{x}_{\text{VM}}^{(t)}, \mathbf{H}_u)$, the embedding $\mathbf{S}_u^{(k)}$ is a deterministic function of $\mathbf{x}_{\text{extra},k}^{(t_1:t_L)}$ alone (since the FM and autoencoder are fixed functions and $\mathbf{H}_u$ supplies the VM features at each time step).
This gives the Markov chain:
\begin{equation}
    \mathbf{x}_{\text{extra},k}^{(t_1:t_L)} \;\to\; \mathbf{E}_u^{(k)} \;\to\; \mathbf{Z}_u^{(k)} \;\to\; \mathbf{S}_u^{(k)}
    \qquad \text{(given } (\mathbf{x}_{\text{VM}}^{(t)}, \mathbf{H}_u)\text{)}.
\end{equation}
Telescope the mutual information with $y$:
\begin{align}
    &\mathcal{I}_{\mathrm{feature\text{-}raw},k} - \mathcal{I}_{\mathrm{cross},k}
    = I(\mathbf{x}_{\text{extra},k}^{(t_1:t_L)}; y \mid \mathbf{x}_{\text{VM}}, \mathbf{H}_u)
     - I(\mathbf{S}_u^{(k)}; y \mid \mathbf{x}_{\text{VM}}, \mathbf{H}_u) \notag \\
    &= \underbrace{\bigl[I(\mathbf{x}_{\text{extra}}; y \mid \cdot) - I(\mathbf{E}; y \mid \cdot)\bigr]}_{\ell_{\mathrm{repr},k}^{\mathrm{cross}}}
     + \underbrace{\bigl[I(\mathbf{E}; y \mid \cdot) - I(\mathbf{Z}; y \mid \cdot)\bigr]}_{\ell_{\mathrm{AE},k}^{\mathrm{cross}}}
     + \underbrace{\bigl[I(\mathbf{Z}; y \mid \cdot) - I(\mathbf{S}; y \mid \cdot)\bigr]}_{\ell_{\mathrm{Q},k}^{\mathrm{cross}}}, \label{eq:cross-telescope-proof}
\end{align}
where $(\cdot) = (\mathbf{x}_{\text{VM}}^{(t)}, \mathbf{H}_u)$ throughout.
Each bracket is non-negative by the data processing inequality: since $\sigma(\mathbf{S}) \subseteq \sigma(\mathbf{Z}) \subseteq \sigma(\mathbf{E}) \subseteq \sigma(\mathbf{x}_{\text{extra}})$ (given $(\mathbf{x}_{\text{VM}}, \mathbf{H}_u)$), we have $I(W; y \mid C) \ge I(f(W); y \mid C)$ for any deterministic $f$ and conditioning $C$.

The cross-platform pipeline identity is exactly~\eqref{eq:cross-telescope-proof}.
Dividing by $\mathcal{I}_{\mathrm{feature\text{-}raw},k}$ (assumed $> 0$) gives $\eta_k = (\ell_{\mathrm{repr},k}^{\mathrm{cross}} + \ell_{\mathrm{AE},k}^{\mathrm{cross}} + \ell_{\mathrm{Q},k}^{\mathrm{cross}}) / \mathcal{I}_{\mathrm{feature\text{-}raw},k}$.
\end{proof}

\begin{lemma}[Temporal retention]\label{lem:tau-pipeline}
$0 \le \mathcal{I}_{\mathrm{residual},k} \le \tau_k \cdot \mathcal{I}_{\mathrm{temporal}}$ where
$\tau_k := (\ell_{\mathrm{repr},k} + \ell_{\mathrm{AE},k} + \ell_{\mathrm{Q},k}) / \mathcal{I}_{\mathrm{temporal}}$.
\end{lemma}

\begin{proof}
This is a restatement of Proposition~\ref{prop:pipeline-decomp} with $\tau_k$ defined as the normalized bound.
From~\eqref{eq:pipeline-bound}:
\begin{equation}
    \mathcal{I}_{\mathrm{residual},k} \;\le\; \ell_{\mathrm{repr},k} + \ell_{\mathrm{AE},k} + \ell_{\mathrm{Q},k}.
\end{equation}
Dividing by $\mathcal{I}_{\mathrm{temporal}}$ (assumed $> 0$) gives $\tau_k := (\ell_{\mathrm{repr},k} + \ell_{\mathrm{AE},k} + \ell_{\mathrm{Q},k}) / \mathcal{I}_{\mathrm{temporal}}$, so $\mathcal{I}_{\mathrm{residual},k} \le \tau_k \cdot \mathcal{I}_{\mathrm{temporal}}$.
\end{proof}

\begin{corollary}[LoopFM gain sandwich]\label{cor:gain-lower-bound}
$(1 - \tau_k)\,\mathcal{I}_{\mathrm{temporal}} + (1 - \eta_k)\,\mathcal{I}_{\mathrm{feature\text{-}raw},k}\leq \mathcal{I}_{\text{LoopFM}}(\text{FM}_k) \leq \mathcal{I}_{\mathrm{temporal}} + (1 - \eta_k)\,\mathcal{I}_{\mathrm{feature\text{-}raw},k}$.
\end{corollary}

\begin{proof}
From Theorem~\ref{thm:gain-decomp}:
$\mathcal{I}_{\text{LoopFM}}(\text{FM}_k) = \mathcal{I}_{\mathrm{temporal}} + \mathcal{I}_{\mathrm{cross},k} - \mathcal{I}_{\mathrm{residual},k}$.

By Lemma~\ref{lem:eta-pipeline}: $\mathcal{I}_{\mathrm{cross},k} = (1 - \eta_k)\,\mathcal{I}_{\mathrm{feature\text{-}raw},k} \ge (1 - \eta_k)\,\mathcal{I}_{\mathrm{feature\text{-}raw},k}$.

By Lemma~\ref{lem:tau-pipeline}: $\mathcal{I}_{\mathrm{residual},k} \le \tau_k\,\mathcal{I}_{\mathrm{temporal}}$, so $-\mathcal{I}_{\mathrm{residual},k} \ge -\tau_k\,\mathcal{I}_{\mathrm{temporal}}$.

Combining:
\begin{equation}
    \mathcal{I}_{\text{LoopFM}}(\text{FM}_k)
    \ge \mathcal{I}_{\mathrm{temporal}} - \tau_k\,\mathcal{I}_{\mathrm{temporal}} + (1 - \eta_k)\,\mathcal{I}_{\mathrm{feature\text{-}raw},k}
    = (1 - \tau_k)\,\mathcal{I}_{\mathrm{temporal}} + (1 - \eta_k)\,\mathcal{I}_{\mathrm{feature\text{-}raw},k}. \qedhere
\end{equation}
\end{proof}

\begin{lemma}[Non-worsening pipeline quality]\label{lem:eta-monotone}
Under \emph{(A2)--(A3)}, $\eta_2 \le \eta_1$.
\end{lemma}

\begin{proof}
By Lemma~\ref{lem:eta-pipeline},
$\eta_k = (\ell_{\mathrm{repr},k}^{\mathrm{cross}} + \ell_{\mathrm{AE},k}^{\mathrm{cross}} + \ell_{\mathrm{Q},k}^{\mathrm{cross}}) / \mathcal{I}_{\mathrm{feature\text{-}raw},k}$.
By Assumption~(A3), Eq.~\eqref{eq:a3-aeq}, the numerator for $k{=}2$ is at most the numerator for $k{=}1$.
By Assumption~(A2) and Lemma~\ref{lem:feature-gap-chain},
$\mathcal{I}_{\mathrm{feature\text{-}raw},2} \ge \mathcal{I}_{\mathrm{feature\text{-}raw},1} > 0$.
Therefore
\[
\eta_2
= \frac{\ell_{\mathrm{repr},2}^{\mathrm{cross}} + \ell_{\mathrm{AE},2}^{\mathrm{cross}} + \ell_{\mathrm{Q},2}^{\mathrm{cross}}}{\mathcal{I}_{\mathrm{feature\text{-}raw},2}}
\;\le\;
\frac{\ell_{\mathrm{repr},1}^{\mathrm{cross}} + \ell_{\mathrm{AE},1}^{\mathrm{cross}} + \ell_{\mathrm{Q},1}^{\mathrm{cross}}}{\mathcal{I}_{\mathrm{feature\text{-}raw},2}}
\;\le\;
\frac{\ell_{\mathrm{repr},1}^{\mathrm{cross}} + \ell_{\mathrm{AE},1}^{\mathrm{cross}} + \ell_{\mathrm{Q},1}^{\mathrm{cross}}}{\mathcal{I}_{\mathrm{feature\text{-}raw},1}}
= \eta_1. \qedhere
\]
\end{proof}

\subsection{Supporting Lemma and Proof for the Transfer-Ratio Theorem} \label{subsec:proof_transfer_ratio_thm}

\paragraph{Proof of the achieved-risk decomposition.}
Let $X:=\mathbf{x}_{\text{VM}}^{(t)}$ and $\mathbf{S}_k:=\mathbf{S}_u^{(k)}$.
Under BCE, the Bayes-optimal risk equals conditional entropy, so
\[
\mathcal{R}^{*(\mathrm{LoopFM}_k)}=H(y\mid X,\mathbf{S}_k)
=H(y\mid X)-I(\mathbf{S}_k;y\mid X)
=\mathcal{R}^{*(\mathrm{VM})}-\mathcal{I}_{\text{LoopFM}}(\text{FM}_k).
\]
By definition,
$R_{\mathrm{ach}}(\mathrm{LoopFM}_k)
=\mathcal{R}^{*(\mathrm{LoopFM}_k)}+\epsilon_{\mathrm{est},k}$
with $\epsilon_{\mathrm{est},k}:=\epsilon_{\mathrm{opt},k}+\epsilon_{\mathrm{gen},k}\ge0$.
Subtracting for $k=1$ and $k=2$:
\[
\Delta_{\mathrm{student}}
=\Bigl(\mathcal{I}_{\text{LoopFM}}(\text{FM}_2)-\mathcal{I}_{\text{LoopFM}}(\text{FM}_1)\Bigr)+\Delta_{\mathrm{est}}.
\]

\begin{lemma}[Feature-gap bounds]\label{lem:feature-gap-chain}
Assume \emph{(A2)}.
The Bayes-risk gap $\Delta_{\text{feat}} := \mathcal{R}^{*(\text{FM}_1)} - \mathcal{R}^{*(\text{FM}_2)}$ and the raw-channel gap $\Delta_{\mathrm{feat\text{-}raw}} := \mathcal{I}_{\mathrm{feature\text{-}raw},2}-\mathcal{I}_{\mathrm{feature\text{-}raw},1}$ both satisfy
\begin{equation}\label{eq:delta-feat-linear-bounds}
\underline{\kappa}_{\mathrm{gap}}\,(m_2-m_1) \le \Delta_{\text{feat}} \le \bar{\kappa}_{\mathrm{gap}}\,(m_2-m_1),
\qquad
\underline{\kappa}_{\mathrm{gap}}^{\mathrm{hist}}\,(m_2-m_1) \le \Delta_{\mathrm{feat\text{-}raw}} \le \bar{\kappa}_{\mathrm{gap}}^{\mathrm{hist}}\,(m_2-m_1).
\end{equation}
Note that $\Delta_{\text{feat}}$ involves current-step features (via the FM's per-sample Bayes risk), while $\Delta_{\mathrm{feat\text{-}raw}}$ involves historical feature sequences (via the Theorem~\ref{thm:gain-decomp} definition $\mathcal{I}_{\mathrm{feature\text{-}raw},k} = I(\mathbf{x}_{\text{extra},k}^{(t_1:t_L)};y\mid\mathbf{x}_{\text{VM}}^{(t)},\mathbf{H}_u)$); these are distinct quantities that admit the same chain-rule structure.
\end{lemma}

\begin{proof}
Both quantities decompose via the MI chain rule over feature groups $j\in\{m_1{+}1,\ldots,m_2\}$.

For the Bayes-risk gap (current-step features):
\begin{align*}
\Delta_{\text{feat}}
&=H\!\left(y\mid \mathbf{x}_{\text{VM}},\mathbf{x}_{\text{extra},1}\right)
-H\!\left(y\mid \mathbf{x}_{\text{VM}},\mathbf{x}_{\text{extra},1},\mathbf{x}_{\text{extra},2}\right)
=\sum_{j=m_1+1}^{m_2} I\!\left(u_j;y\mid \mathbf{x}_{\text{VM}},\mathbf{x}_{\text{extra},1},u_{m_1+1:j-1}\right).
\end{align*}

For the raw-channel gap (historical sequences, using the Theorem~\ref{thm:gain-decomp} definition):
\begin{align*}
\Delta_{\mathrm{feat\text{-}raw}}
&=I\!\left(\mathbf{x}_{\text{extra},2}^{(t_1:t_L)};y\mid \mathbf{x}_{\text{VM}}^{(t)},\mathbf{H}_u,\mathbf{x}_{\text{extra},1}^{(t_1:t_L)}\right)
=\sum_{j=m_1+1}^{m_2} I\!\left(u_j^{(t_1:t_L)};y\mid \mathbf{x}_{\text{VM}}^{(t)},\mathbf{H}_u,\mathbf{x}_{\text{extra},1}^{(t_1:t_L)},u_{m_1+1:j-1}^{(t_1:t_L)}\right).
\end{align*}

Applying (A2) to each per-feature-group term in both expansions and summing gives~\eqref{eq:delta-feat-linear-bounds}.
In particular, $\mathcal{I}_{\mathrm{feature\text{-}raw},2} \ge \mathcal{I}_{\mathrm{feature\text{-}raw},1} + \underline{\kappa}_{\mathrm{gap}}^{\mathrm{hist}}(m_2-m_1) > \mathcal{I}_{\mathrm{feature\text{-}raw},1}$.
\end{proof}

\subsection{Proof of Theorem~\ref{thm:loopfm-tr} (Transfer-Ratio Bound)}
\label{app:proof-tr}

\begin{proof}
Let
\[
\Delta_{\text{LoopFM}} := R_{\text{ach}}(\text{LoopFM}_1)-R_{\text{ach}}(\text{LoopFM}_2),
\qquad
\Delta_{\text{teacher}} := R_{\text{ach}}(\text{FM}_1)-R_{\text{ach}}(\text{FM}_2).
\]
By Eq.~\eqref{eq:tr-def},
\[
\mathrm{TR}_{\text{LoopFM}}=\frac{\Delta_{\text{LoopFM}}}{\Delta_{\text{teacher}}}.
\]
By the achieved-risk decomposition,
\[
\Delta_{\text{LoopFM}}
=\Bigl(\mathcal{I}_{\text{LoopFM}}(\text{FM}_2)-\mathcal{I}_{\text{LoopFM}}(\text{FM}_1)\Bigr)+\Delta_{\text{est}}.
\]
Under the well-trained assumption $\epsilon_{\mathrm{est},k}\to0$, $\Delta_{\text{est}}\to0$, so
\[
\Delta_{\text{LoopFM}}
=\mathcal{I}_{\text{LoopFM}}(\text{FM}_2)-\mathcal{I}_{\text{LoopFM}}(\text{FM}_1).
\]
By Eq.~\eqref{eq:teacher-decomp},
\[
\Delta_{\text{teacher}}
=\Delta_{\text{feat}}+\Delta_{\text{param}}
=\mathcal{R}^{*(\text{FM}_1)} - \mathcal{R}^{*(\text{FM}_2)} + \epsilon_{\text{over}}(p_1,m_1)-\epsilon_{\text{over}}(p_2,m_2).
\]
Substituting into Eq.~\eqref{eq:tr-def} gives the transfer-ratio expression.

Assume now $\Delta_{\text{teacher}}>0$.

\textbf{(0) Initial launch.}
When the VM has no prior LoopFM, $R_{\mathrm{ach}}(\text{LoopFM}_1) = R_{\mathrm{ach}}(\text{VM})$ and $\mathcal{I}_{\text{LoopFM}}(\text{FM}_1) = 0$ (the baseline VM receives no embedding sequence). The numerator becomes $\mathcal{I}_{\text{LoopFM}}(\text{FM}_2)$. This quantity is non-negative by conditioning reduces entropy:
\[
\mathcal{R}^{*(\text{LoopFM}_2)} = H(y \mid \mathbf{x}_{\text{VM}}^{(t)}, \mathbf{S}_u^{(2)}) \;\le\; H(y \mid \mathbf{x}_{\text{VM}}^{(t)}) = \mathcal{R}^{*(\text{VM})},
\]
since the LoopFM$_2$ model observes $(\mathbf{x}_{\text{VM}}^{(t)}, \mathbf{S}_u^{(2)})$, a strict superset of the baseline VM's input $\mathbf{x}_{\text{VM}}^{(t)}$.
Therefore $\mathcal{I}_{\text{LoopFM}}(\text{FM}_2) = \mathcal{R}^{*(\text{VM})} - \mathcal{R}^{*(\text{LoopFM}_2)} \ge 0$.
Since $\Delta_{\text{teacher}} > 0$ by assumption, $\mathrm{TR}_{\text{LoopFM}} \ge 0$.

For the tighter bound, substitute the gain sandwich lower bound (Corollary~\ref{cor:gain-lower-bound}):
\[
\mathcal{I}_{\text{LoopFM}}(\text{FM}_2) \;\ge\; (1-\tau_2)\,\mathcal{I}_{\mathrm{temporal}} + (1-\eta_2)\,\mathcal{I}_{\mathrm{feature\text{-}raw},2}.
\]
Dividing by $\Delta_{\text{teacher}} > 0$ gives~\eqref{eq:tr-launch}. Note that $\mathrm{TR}_{\text{LoopFM}} \ge 0$ follows independently from the Bayes-risk argument above, not from the sign of this bound.

\textbf{(1) Negative transfer (without A3).}
If $\Delta_{\text{LoopFM}}<0$, then from
\(
\mathrm{TR}_{\text{LoopFM}}=\Delta_{\text{LoopFM}}/\Delta_{\text{teacher}}
\)
we immediately get $\mathrm{TR}_{\text{LoopFM}}<0$.

For a sufficient condition, Corollary~\ref{cor:gain-lower-bound} gives
\[
\mathcal{I}_{\text{LoopFM}}(\text{FM}_2)
\le
\mathcal{I}_{\mathrm{temporal}} + (1-\eta_2)\,\mathcal{I}_{\mathrm{feature\text{-}raw},2},
\]
and
\[
\mathcal{I}_{\text{LoopFM}}(\text{FM}_1)
\ge
(1-\tau_1)\,\mathcal{I}_{\mathrm{temporal}} + (1-\eta_1)\,\mathcal{I}_{\mathrm{feature\text{-}raw},1}.
\]
Therefore,
\begin{align}
\mathcal{I}_{\text{LoopFM}}(\text{FM}_2)-\mathcal{I}_{\text{LoopFM}}(\text{FM}_1)
&\le
\tau_1\,\mathcal{I}_{\mathrm{temporal}} + (1-\eta_2)\,\mathcal{I}_{\mathrm{feature\text{-}raw},2}
- (1-\eta_1)\,\mathcal{I}_{\mathrm{feature\text{-}raw},1}. \label{eq:loopfm-num-upper}
\end{align}
So if the right-hand side is negative, then $\Delta_{\text{LoopFM}}<0$ and $\mathrm{TR}_{\text{LoopFM}}<0$.

\emph{Why this requires violating~(A3).} Under~(A3), Lemma~\ref{lem:eta-monotone} gives $\eta_2 \le \eta_1$, so $(1{-}\eta_2) \ge (1{-}\eta_1)$. Combined with~(A2) ($\mathcal{I}_{\mathrm{feature\text{-}raw},2} \ge \mathcal{I}_{\mathrm{feature\text{-}raw},1}$), the right-hand side of~\eqref{eq:loopfm-num-upper} satisfies
\[
\tau_1\,\mathcal{I}_{\mathrm{temporal}} + (1{-}\eta_2)\,\mathcal{I}_{\mathrm{feature\text{-}raw},2} - (1{-}\eta_1)\,\mathcal{I}_{\mathrm{feature\text{-}raw},1}
\;\ge\; \tau_1\,\mathcal{I}_{\mathrm{temporal}} \;\ge\; 0,
\]
so the negative-transfer condition is never met. Therefore, negative transfer requires $\eta_2 > \eta_1$ (i.e., violating~A3).

\textbf{(2) Positive transfer (with A3).}
By Assumption~(A3) and Lemma~\ref{lem:eta-monotone}, $\eta_2\le\eta_1$.
For the numerator lower bound, Corollary~\ref{cor:gain-lower-bound} implies
\[
\mathcal{I}_{\text{LoopFM}}(\text{FM}_2)
\ge
(1-\tau_2)\,\mathcal{I}_{\mathrm{temporal}} + (1-\eta_2)\,\mathcal{I}_{\mathrm{feature\text{-}raw},2},
\]
and
\[
\mathcal{I}_{\text{LoopFM}}(\text{FM}_1)
\le
\mathcal{I}_{\mathrm{temporal}} + (1-\eta_1)\,\mathcal{I}_{\mathrm{feature\text{-}raw},1}.
\]
Therefore
\begin{align}
\mathcal{I}_{\text{LoopFM}}(\text{FM}_2)-\mathcal{I}_{\text{LoopFM}}(\text{FM}_1)
&\ge -\tau_2\,\mathcal{I}_{\mathrm{temporal}}
+ (1-\eta_2)\,\mathcal{I}_{\mathrm{feature\text{-}raw},2}
- (1-\eta_1)\,\mathcal{I}_{\mathrm{feature\text{-}raw},1}. \label{eq:loopfm-num-step1}
\end{align}
Under $\eta_2\le\eta_1$,
\begin{align}
(1-\eta_2)\,\mathcal{I}_{\mathrm{feature\text{-}raw},2}&
- (1-\eta_1)\,\mathcal{I}_{\mathrm{feature\text{-}raw},1} \nonumber\\
&=
(1-\eta_1)\bigl(\mathcal{I}_{\mathrm{feature\text{-}raw},2}-\mathcal{I}_{\mathrm{feature\text{-}raw},1}\bigr)
+ (\eta_1-\eta_2)\,\mathcal{I}_{\mathrm{feature\text{-}raw},2} \notag\\
&\ge
(1-\eta_1)\bigl(\mathcal{I}_{\mathrm{feature\text{-}raw},2}-\mathcal{I}_{\mathrm{feature\text{-}raw},1}\bigr) \notag\\
&=
(1-\eta_1)\,\Delta_{\mathrm{feat\text{-}raw}}. \label{eq:loopfm-num-step2}
\end{align}
By Lemma~\ref{lem:feature-gap-chain},
\[
\Delta_{\mathrm{feat\text{-}raw}}
=\mathcal{I}_{\mathrm{feature\text{-}raw},2}-\mathcal{I}_{\mathrm{feature\text{-}raw},1}
\ge \underline{\kappa}_{\mathrm{gap}}^{\mathrm{hist}}(m_2-m_1).
\]
Combining with~\eqref{eq:loopfm-num-step1}--\eqref{eq:loopfm-num-step2},
\[
\Delta_{\text{LoopFM}}
\ge
-\tau_2\,\mathcal{I}_{\mathrm{temporal}}
+ (1-\eta_1)\,\underline{\kappa}_{\mathrm{gap}}^{\mathrm{hist}}(m_2-m_1).
\]

For the denominator upper bound, by Lemma~\ref{lem:feature-gap-chain}:
\[
\Delta_{\text{feat}}\le \bar{\kappa}_{\mathrm{gap}}(m_2-m_1).
\]
By the two-sided Assumption~(A1), Eq.~\eqref{eq:a1-envelope}:
\[
\underline C_{\text{over}}\,\sigma^2\,\xi_k
\;\le\;
\epsilon_{\text{over}}(p_k,m_k)
\;\le\;
\bar C_{\text{over}}\,\sigma^2\,\xi_k,\quad k\in\{1,2\}.
\]
To rigorously upper-bound the difference $\Delta_{\text{param}}=\epsilon_{\text{over}}(p_1,m_1)-\epsilon_{\text{over}}(p_2,m_2)$, we apply the \emph{upper} envelope to $\epsilon_{\text{over}}(p_1,m_1)$ and the \emph{lower} envelope to $\epsilon_{\text{over}}(p_2,m_2)$:
\[
\Delta_{\text{param}}
\;\le\;
\bar C_{\text{over}}\,\sigma^2\,\xi_1 - \underline C_{\text{over}}\,\sigma^2\,\xi_2.
\]
With $\bar\kappa_{\mathrm{over}} := \bar C_{\mathrm{over}}\,\sigma^2$ and $\underline\kappa_{\mathrm{over}} := \underline C_{\mathrm{over}}\,\sigma^2$:
\[
\Delta_{\text{param}}
\;\le\;
\bar\kappa_{\mathrm{over}}\,\xi_1
- \underline\kappa_{\mathrm{over}}\,\xi_2.
\]
Hence
\[
\Delta_{\text{feat}}+\Delta_{\text{param}}
\le
\bar{\kappa}_{\mathrm{gap}}(m_2-m_1)
+ \bar\kappa_{\mathrm{over}}\,\xi_1
- \underline\kappa_{\mathrm{over}}\,\xi_2.
\]
If additionally
\[
- \tau_2\,\mathcal{I}_{\mathrm{temporal}} + (1 - \eta_1) \underline{\kappa}_{\mathrm{gap}}^{\mathrm{hist}}(m_2-m_1)\ge0,
\]
then
\[
\mathrm{TR}_{\text{LoopFM}}
=\frac{\Delta_{\text{LoopFM}}}{\Delta_{\text{feat}}+\Delta_{\text{param}}}
\ge
\frac{ - \tau_2\,\mathcal{I}_{\mathrm{temporal}} + (1 - \eta_1) \underline{\kappa}_{\mathrm{gap}}^{\mathrm{hist}}(m_2-m_1)}
{\bar{\kappa}_{\mathrm{gap}}(m_2-m_1) + \bar\kappa_{\mathrm{over}}\,\xi_1 - \underline\kappa_{\mathrm{over}}\,\xi_2},
\]
which is Eq.~\eqref{eq:loopfm-tr-lb}.
\end{proof}

\subsection{Proof of Corollary~\ref{cor:tr-monotone} (Monotonicity in Feature Gap)}
\label{app:proof-tr-monotone}

\begin{proof}
Write $\delta:=m_2-m_1$ and define
\[
N(\delta):= a\delta + b,\qquad
D(\delta):= \bar{\kappa}_{\mathrm{gap}}\,\delta
+ \underbrace{\bar\kappa_{\mathrm{over}}\,\xi_1 - \underline\kappa_{\mathrm{over}}\,\xi_2(\delta)}_{=:\,\epsilon(\delta)},
\]
where $a:=(1-\eta_1)\underline{\kappa}_{\mathrm{gap}}^{\mathrm{hist}}>0$, $b:=-\tau_2\mathcal{I}_{\mathrm{temporal}}\le0$, and $\xi_2(\delta):=(m_1+\delta)/n + n/(p_2-m_1-\delta)$.
Then $\mathrm{TR}_{\text{LB}}(\delta)=N(\delta)/D(\delta)$.

\textbf{Step 1: Sign of $D'(\delta)$.}
$D'(\delta)=\bar{\kappa}_{\mathrm{gap}} + \epsilon'(\delta) = \bar{\kappa}_{\mathrm{gap}} - \underline\kappa_{\mathrm{over}}\xi_2'(\delta)$,
where $\xi_2'(\delta)=1/n+n/(p_2{-}m_1{-}\delta)^2>0$.
In the overparameterized regime ($p_2-m_1-\delta\gg\sqrt{n}$), $\underline\kappa_{\mathrm{over}}\xi_2'(\delta)\ll\bar{\kappa}_{\mathrm{gap}}$, so $D'(\delta)>0$.

\textbf{Step 2: Decomposition of $g(\delta):=aD(\delta)-N(\delta)D'(\delta)$.}
Expanding:
\begin{align}
g(\delta)
&= a[\bar{\kappa}_{\mathrm{gap}}\delta+\epsilon(\delta)] - (a\delta+b)[\bar{\kappa}_{\mathrm{gap}}+\epsilon'(\delta)] \notag\\
&= a[\epsilon(\delta)-\delta\epsilon'(\delta)] + |b|[\bar{\kappa}_{\mathrm{gap}}+\epsilon'(\delta)] \notag\\
&= a[\epsilon(\delta)-\delta\epsilon'(\delta)] + |b|\,D'(\delta). \label{eq:g-decomp}
\end{align}

\textbf{Step 3: The second term $|b|\,D'(\delta)\ge 0$.}
Immediate from Step~1 ($D'>0$) and $|b|\ge 0$.

\textbf{Step 4: The first term $a[\epsilon(\delta)-\delta\epsilon'(\delta)]\ge 0$.}
Since $\epsilon'(\delta)=-\underline\kappa_{\mathrm{over}}\xi_2'(\delta)<0$, we have $-\delta\epsilon'(\delta)>0$, so it suffices to show $\epsilon(\delta)-\delta\epsilon'(\delta)\ge 0$.
Consider the function $h(\delta):=\xi_2(\delta)-\delta\xi_2'(\delta)$.
Then $\epsilon-\delta\epsilon' = \bar\kappa_{\mathrm{over}}\xi_1 - \underline\kappa_{\mathrm{over}}h(\delta)$, so we need $\bar\kappa_{\mathrm{over}}\xi_1\ge\underline\kappa_{\mathrm{over}}h(\delta)$.
Differentiating: $h'(\delta)=-\delta\xi_2''(\delta)$, and $\xi_2''(\delta)=2n/(p_2{-}m_1{-}\delta)^3>0$, so $h'(\delta)=-\delta\xi_2''(\delta)\le 0$.
Hence $h(\delta)\le h(0)=\xi_2(0)=m_1/n + n/(p_2-m_1)$ for all $\delta\ge 0$.
Since $p_2\ge p_1$: $\xi_2(0)=m_1/n+n/(p_2-m_1)\le m_1/n+n/(p_1-m_1)=\xi_1$.
Since $\bar\kappa_{\mathrm{over}}=\bar C_{\mathrm{over}}\sigma^2\ge\underline C_{\mathrm{over}}\sigma^2=\underline\kappa_{\mathrm{over}}$:
\[
\underline\kappa_{\mathrm{over}}h(\delta)\le\underline\kappa_{\mathrm{over}}\xi_2(0)\le\underline\kappa_{\mathrm{over}}\xi_1\le\bar\kappa_{\mathrm{over}}\xi_1.
\]
Therefore $\epsilon(\delta)-\delta\epsilon'(\delta)\ge 0$.

\textbf{Step 5: Conclusion.}
From~\eqref{eq:g-decomp}, $g(\delta)\ge 0$ for all $\delta\ge 0$, with strict inequality whenever $b<0$ or $\epsilon(\delta)>0$.
Since $\mathrm{TR}_{\text{LB}}'(\delta)=g(\delta)/D(\delta)^2\ge 0$, $\mathrm{TR}_{\text{LB}}$ is non-decreasing (and strictly increasing under mild non-degeneracy).

\textbf{Asymptote.}
As $\delta\to\infty$ (with $p_2\gg m_1+\delta$ maintained so that $\xi_2(\delta)/\delta\to 0$), $\epsilon(\delta)/\delta\to 0$, so $D(\delta)\sim\bar{\kappa}_{\mathrm{gap}}\,\delta$ and
\[
\mathrm{TR}_{\text{LB}}(\delta)\;\to\;\frac{a}{\bar{\kappa}_{\mathrm{gap}}}=\frac{(1-\eta_1)\underline{\kappa}_{\mathrm{gap}}^{\mathrm{hist}}}{\bar{\kappa}_{\mathrm{gap}}}\;>\;0.
\]
Convergence is from below since $\mathrm{TR}_{\text{LB}}$ is increasing.
\end{proof}

\subsection{Sequence Length Analysis}
\label{app:seq-length-theory}

We now make the dependence on the sequence length $L$ explicit and show that longer sequences monotonically increase the LoopFM information gain. Write $\mathbf{H}_u^{(L)} := (\mathbf{x}_{\text{VM}}^{(t_1)}, \ldots, \mathbf{x}_{\text{VM}}^{(t_L)})$ and $\mathbf{S}_u^{(k,L)} := [\mathbf{z}_{k,t_1}, \ldots, \mathbf{z}_{k,t_L}]$ with the three information quantities now $L$-dependent: $\mathcal{I}_{\mathrm{temporal}}(L)$, $\mathcal{I}_{\mathrm{cross},k}(L)$, $\mathcal{I}_{\mathrm{residual},k}(L)$.

\begin{proposition}[Monotonicity in sequence length]\label{prop:monotone-L}
For all $L \ge 1$:
\begin{enumerate}[label=(\roman*), itemsep=3pt]
    \item \textbf{Temporal information is non-decreasing:}
    \begin{equation}\label{eq:temporal-monotone}
        \mathcal{I}_{\mathrm{temporal}}(L+1) = \mathcal{I}_{\mathrm{temporal}}(L) + \underbrace{I\!\left(\mathbf{x}_{\text{VM}}^{(t_{L+1})};\; y \;\middle|\; \mathbf{x}_{\text{VM}}^{(t)},\, \mathbf{H}_u^{(L)}\right)}_{\delta_{\mathrm{temp}}(L) \;\ge\; 0} \ge \mathcal{I}_{\mathrm{temporal}}(L).
    \end{equation}
    \item \textbf{Total LoopFM gain is non-decreasing:}
    \begin{equation}\label{eq:loopfm-monotone}
        \mathcal{I}_{\text{LoopFM},k}(L+1) = \mathcal{I}_{\text{LoopFM},k}(L) + \underbrace{I\!\left(\mathbf{z}_{k,t_{L+1}};\; y \;\middle|\; \mathbf{x}_{\text{VM}}^{(t)},\, \mathbf{S}_u^{(k,L)}\right)}_{\delta_{\text{LoopFM},k}(L) \;\ge\; 0} \ge \mathcal{I}_{\text{LoopFM},k}(L),
    \end{equation}
    where $\mathcal{I}_{\text{LoopFM},k}(L) := I(\mathbf{S}_u^{(k,L)}; y \mid \mathbf{x}_{\text{VM}}^{(t)})$.
    \item \textbf{Convergence:} Both sequences converge: $\mathcal{I}_{\mathrm{temporal}}(L) \nearrow \mathcal{I}_{\mathrm{temporal}}^*$ and $\mathcal{I}_{\text{LoopFM},k}(L) \nearrow \mathcal{I}_{\text{LoopFM},k}^*$ as $L \to \infty$, since they are non-decreasing and bounded above by $H(y \mid \mathbf{x}_{\text{VM}}^{(t)}) \le \log 2$.
\end{enumerate}
Both monotonicity results follow from the chain rule of MI, which decomposes the information in $L{+}1$ events into the first $L$ events plus a non-negative marginal contribution from the $(L{+}1)$-th event.
\end{proposition}


\begin{remark}[Sequence length and achieved risk]
At the population level, Proposition~\ref{prop:monotone-L} guarantees that longer sequences always improve LoopFM. Since the total gain is bounded by $H(y)$, the marginal contribution $\delta_{\text{LoopFM},k}(L)\to 0$, yielding diminishing returns. For achieved risk, there is a bias--variance trade-off: larger $L$ increases information gain but can increase estimation cost, so there exists an optimal $L^*$ beyond which estimation cost dominates---though in practice, Transformer-based encoders handle long sequences efficiently, making $L^*$ large. This is consistent with Table~\ref{tab:ablation}, which shows continued gains up to $L = 100{+}$.
\end{remark}


\subsection{Derivation of the A1 Excess-Risk Envelope from NTK + Benign Overfitting Assumptions}
\label{app:overparameterization}

Rather than postulating the A1 envelope directly, we derive it from: (i) NTK/lazy-training linearization~\citep{jacot2018neural}, and (ii) the finite-sample benign-overfitting conditions in \citet[Theorem~1]{bartlett2020benign}.

\paragraph{Applicability to cross-entropy training.}
\citet{bartlett2020benign} analyzes the minimum-norm interpolating (MNI) estimator under squared loss, whereas our FMs are trained with binary cross-entropy (BCE).
\citet{wang2023benign} bridge this gap: in the overparameterized regime ($p_k\gg n$), gradient descent on CE loss converges to the SVM solution, which equals the MNI solution~\citep[Theorem~1, Corollary~1]{wang2023benign}---the predictor $\hat{w}$ and classification error are identical regardless of training loss. Bartlett's excess-risk bounds on the MNI therefore apply directly to CE-trained FMs.
\citet[Theorems~4--5]{wang2023benign} further establish that benign overfitting extends to classification with the same spectral conditions, confirming the applicability of our assumptions.

\paragraph{Notation alignment with \citet{bartlett2020benign}.}
Their linear model is $y=\langle w^\star,x\rangle+\xi$ with covariance matrix $\Sigma:=\mathbb{E}[xx^\top]$ and eigenvalues $(\lambda_j)_j$.
In our linearized FM surrogate, we identify:
\begin{itemize}[leftmargin=1.5em, itemsep=1pt]
    \item $x \leftrightarrow \phi_k(u,a)\in\mathbb{R}^{p_k}$ (NTK/lazy-training feature of FM$_k$),
    \item $w^\star \leftrightarrow w_k^\star$,
    \item $\Sigma \leftrightarrow \Sigma_k:=\mathbb{E}[\phi_k\phi_k^\top]$,
    \item prediction error $\epsilon_{\text{over}}(p_k,m_k) \leftrightarrow$ excess risk term induced by finite-width/finite-sample interpolation.
\end{itemize}

\paragraph{Assumptions (Bartlett-style specialization).}
Let $\lambda_{k,1}\ge\cdots\ge\lambda_{k,p_k}\ge 0$ be the eigenvalues of $\Sigma_k$. We assume:
\begin{enumerate}[label=(B\arabic*), leftmargin=1.6em, itemsep=1pt]
    \item \textbf{Linearized regime:} FM$_k$ operates in a local NTK/lazy-training regime so that prediction is modeled by a linear rule on $\phi_k$~\citep{jacot2018neural}.
    \item \textbf{Data model:} $y=\langle w_k^\star,\phi_k\rangle+\xi$, with $\mathbb{E}[\xi\mid\phi_k]=0$ and $\mathbb{E}[\xi^2\mid\phi_k]\le\sigma^2$.
    \item \textbf{Bartlett spectral split condition:} with split index $d_k$, the quantities
    \[
    r_{d_k}(\Sigma_k):=\frac{\sum_{j>d_k}\lambda_{k,j}}{\lambda_{k,d_k+1}},
    \qquad
    R_{d_k}(\Sigma_k):=\frac{\bigl(\sum_{j>d_k}\lambda_{k,j}\bigr)^2}{\sum_{j>d_k}\lambda_{k,j}^2}
    \]
    satisfy the finite-sample side conditions required by \citet[Theorem~1]{bartlett2020benign}.
\end{enumerate}

Here, the \emph{split index} $d_k\in\{1,\ldots,p_k-1\}$ is the cutoff used in \citet{bartlett2020benign} to separate the spectrum into a leading block and a tail block: indices $j\le d_k$ are the principal/signal directions, while indices $j>d_k$ form the overparameterized tail that governs interpolation noise.

\begin{theorem}[Linearized-surrogate excess-risk scaling]\label{thm:linearized-overparam}
Under \textnormal{(B1)--(B3)} and the distributional/moment conditions of
\citet[Theorem~1]{bartlett2020benign}, let $n$ denote the number of training samples. The expected excess risk of FM$_k$ satisfies
\begin{equation}\label{eq:linearized-overparam-bound}
\underline C_{\mathrm{BO}}\,\sigma^2\,\Psi_k\!\left(\Sigma_k,d_k\right)
\;\le\;
\epsilon_{\text{over}}(p_k,d_k)
\;\le\;
\bar C_{\mathrm{BO}}\,\sigma^2\,\Psi_k\!\left(\Sigma_k,d_k\right),
\end{equation}
for constants $\underline C_{\mathrm{BO}},\bar C_{\mathrm{BO}}>0$, where the noise functional is
\begin{equation}\label{eq:Psi-explicit}
\Psi_k(\Sigma_k,d_k)
:=
\frac{d_k}{n}
\;+
\frac{n}{R_{d_k}(\Sigma_k)},
\end{equation}
with $\widetilde r_{d_k}(\Sigma_k):=\sum_{j>d_k}\lambda_{k,j}$ and $R_{d_k}(\Sigma_k)$ as defined above.
The two-sided bound holds under the benign-regime condition $r_0(\Sigma_k)/n \to 0$, which ensures the signal (bias) term is dominated by the noise term.
\end{theorem}

\begin{proof}
By (B1)--(B3), the aligned surrogate pair $(w_k^\star,\Sigma_k)$ is in the linear-model setting of
\citet[Theorem~1]{bartlett2020benign}.

\textbf{Lower bound (expected risk).}
Bartlett's Theorem~1 gives $\mathbb{E}[R(\hat\theta)] \ge (\sigma^2/c)(d_k/n + n/R_{d_k})$, establishing the lower side of~\eqref{eq:linearized-overparam-bound}.

\textbf{Upper bound (high-probability to expectation).}
Bartlett's Theorem~1 gives, with probability $\ge 1-\delta$ (for $\log(1/\delta)<n/c$):
\[
R(\hat\theta) \le \underbrace{c\!\left(\|w_k^\star\|^2\|\Sigma_k\|\max\!\left(\sqrt{\frac{r_0}{n}},\frac{r_0}{n},\sqrt{\frac{\log(1/\delta)}{n}}\right)\right)}_{=:\,S(\delta)}
+ c\log(1/\delta)\,\sigma^2\,\Psi_k.
\]
Set $t:=c\log(1/\delta)\sigma^2\Psi_k$, so $\delta=\exp(-t/(c\sigma^2\Psi_k))$, giving $P(R(\hat\theta)-S(\delta)>t)\le e^{-t/(c\sigma^2\Psi_k)}$.
Tail-integrating:
\[
\mathbb{E}[R(\hat\theta)]
\;\le\; S_0 + \int_0^{\infty} e^{-t/(c\sigma^2\Psi_k)}\,dt
\;=\; S_0 + c\,\sigma^2\,\Psi_k,
\]
where $S_0:=S(e^{-1})=O(\|w_k^\star\|^2\|\Sigma_k\|\sqrt{r_0/n})$ is the signal term evaluated at $\delta=e^{-1}$.
Under bounded signal-to-noise ratio $\|w_k^\star\|^2\|\Sigma_k\|/\sigma^2=O(1)$, we have $S_0=O(\sigma^2\sqrt{r_0/n})=O(\sigma^2/\sqrt{n})$, while $\sigma^2\Psi_k\ge \sigma^2 n/R_{d_k}$.
Hence $S_0/(\sigma^2\Psi_k)=O(R_{d_k}/n^{3/2})\to 0$ as $n\to\infty$, so $S_0=o(\sigma^2\Psi_k)$ is absorbed into the noise term, yielding the upper side of~\eqref{eq:linearized-overparam-bound}.
\end{proof}

\begin{corollary}[Recovering the two-sided A1 envelope from Bartlett]\label{cor:a1-from-linearized}
Under \textnormal{(B1)--(B2)}, Theorem~\ref{thm:linearized-overparam} yields the two-sided Assumption~(A1) with $\xi_k := \Psi_k(\Sigma_k,d_k) = d_k/n + n/R_{d_k}(\Sigma_k)$.
For $\xi_k\to 0$ (i.e., $\epsilon_{\text{over}}\to 0$), we adopt the benign covariance condition of \citet[Definition~4]{bartlett2020benign}: the covariance sequence $\Sigma_k$ (indexed by $n$) is \emph{benign} if
\[
\frac{r_0(\Sigma_k)}{n}\to 0,\qquad \frac{d_k}{n}\to 0,\qquad \frac{n}{R_{d_k}(\Sigma_k)}\to 0,
\]
where $d_k$ is the split index (corresponding to $k^*_n$ in Bartlett's notation).
These three conditions ensure (i)~the signal term $S_0\to 0$ (via $r_0/n\to 0$), and (ii)~the noise functional $\xi_k = d_k/n + n/R_{d_k}\to 0$.
\citet[Theorem~2.2]{bartlett2020benign} shows these conditions are satisfied for finite-dimensional models ($p_k\gg n$) with a bi-level covariance: $d_k$ large eigenvalues plus $p_k-d_k$ small, nearly-equal eigenvalues at level $\lambda_{\mathrm{tail}}$, giving $R_{d_k}\asymp p_k-d_k\gg n$.
This corresponds to $d_k\leftrightarrow m_k$ in the main-text notation.
\end{corollary}

\begin{proof}
\textbf{Step 1 (split index choice).}
We set $d_k := m_k$, the input feature dimension of FM$_k$. In the NTK linearization, the $p_k$-dimensional feature map $\phi_k$ has $m_k$ directions aligned with the input features (the signal subspace) and $p_k - m_k$ directions from overparameterization (the interpolation tail). This is the natural Bartlett split: the first $m_k$ eigendirections of $\Sigma_k$ capture input signal, and the remaining $p_k - m_k$ form the tail whose effective rank $R_{m_k}$ governs the noise functional.

\textbf{Step 2 (two-sided bound).}
With $d_k = m_k$, Theorem~\ref{thm:linearized-overparam} gives $\underline C_{\mathrm{BO}}\sigma^2\Psi_k \le \epsilon_{\text{over}} \le \bar C_{\mathrm{BO}}\sigma^2\Psi_k$ where $\Psi_k = m_k/n + n/R_{m_k}$.
Setting $\xi_k := \Psi_k$, $\underline C_{\mathrm{over}} := \underline C_{\mathrm{BO}}$, $\bar C_{\mathrm{over}} := \bar C_{\mathrm{BO}}$ recovers the A1 envelope.

\textbf{Step 3 (benign regime).}
Following \citet[Definition~4]{bartlett2020benign}, the covariance $\Sigma_k$ is benign if three conditions hold as $p_k, n \to \infty$:
\begin{enumerate}[label=(\roman*), itemsep=2pt]
\item $r_0(\Sigma_k)/n \to 0$: The full effective rank $r_0 = \mathrm{tr}(\Sigma_k)/\|\Sigma_k\|$ grows slower than $n$.
This ensures the signal (bias) term $S_0 = O(\|w_k^\star\|^2\|\Sigma_k\|\sqrt{r_0/n})$ is absorbed into the noise term (see Theorem~\ref{thm:linearized-overparam} proof).
\item $m_k/n \to 0$: The split index (number of signal directions) is sublinear in the sample size.
This ensures the first term of $\xi_k = m_k/n + n/R_{m_k}$ vanishes.
\item $n/R_{m_k}(\Sigma_k) \to 0$: The tail effective rank $R_{m_k}$ exceeds the sample size.
We show $R_{m_k} = p_k - m_k$ under a bi-level covariance where the $p_k - m_k$ tail eigenvalues are all equal to $\lambda_{\mathrm{tail}} > 0$.
Let $q := p_k - m_k$. Then:
\[
\sum_{j>m_k}\lambda_j = q\,\lambda_{\mathrm{tail}},\qquad
\sum_{j>m_k}\lambda_j^2 = q\,\lambda_{\mathrm{tail}}^2.
\]
By definition of the effective rank:
\[
R_{m_k} = \frac{\bigl(\sum_{j>m_k}\lambda_j\bigr)^2}{\sum_{j>m_k}\lambda_j^2}
= \frac{(q\,\lambda_{\mathrm{tail}})^2}{q\,\lambda_{\mathrm{tail}}^2}
= \frac{q^2\,\lambda_{\mathrm{tail}}^2}{q\,\lambda_{\mathrm{tail}}^2}
= q = p_k - m_k.
\]
Since $p_k \gg n$: $R_{m_k} = p_k - m_k \gg n$, so $n/R_{m_k} \to 0$.
\end{enumerate}
Together, (i)--(iii) give $\xi_k = m_k/n + n/R_{m_k} \to 0$ and the signal absorption holds, yielding $\epsilon_{\text{over}} \to 0$.
\end{proof}

\paragraph{Spectral validation.}
In overparameterized FMs ($p_k \gg n$), the NTK covariance $\Sigma_k$ exhibits a spiked structure: a few large eigenvalues (signal directions aligned with the $m_k$ input features) followed by a large bulk of $p_k - m_k$ small, approximately equal eigenvalues.
This matches the bi-level structure required by the benign condition of \citet{bartlett2020benign}: the $p_k-m_k$ tail eigenvalues form an approximately flat bulk, giving $R_{m_k}\asymp p_k-m_k\gg n$.

\paragraph{Interpretation for FM comparisons.}
FM$_1$ and FM$_2$ are trained on the same pipeline/distribution, so multiplicative constants can be shared; FM-specific variation is isolated in
$\xi_k = d_k/n + n/(p_k - d_k)$.
Hence the parameter term in Eq.~\eqref{eq:teacher-decomp} is driven by differences in sample-to-dimension ratios (split index vs.\ overparameterized width), not just raw parameter count.

\subsection{General Transfer-Ratio Bound via Bartlett's Effective Rank}
\label{app:general-tr-bound}

The transfer-ratio bound in Theorem~\ref{thm:loopfm-tr} uses $\xi_k = m_k/n + n/(p_k - m_k)$, which assumes a bi-level covariance with $R_{m_k}(\Sigma_k) = p_k - m_k$. This can be relaxed to hold under any covariance satisfying the benign conditions of \citet[Definition~4]{bartlett2020benign}.

\begin{theorem}[General transfer-ratio bound]\label{thm:loopfm-tr-general}
Under the same conditions as Theorem~\ref{thm:loopfm-tr}, but replacing (A1) with the general Bartlett noise functional at split index $d_k$:
\[
\xi_k := \frac{d_k}{n} + \frac{n}{R_{d_k}(\Sigma_k)},
\qquad
R_{d_k}(\Sigma_k) := \frac{\bigl(\sum_{j>d_k}\lambda_{k,j}\bigr)^2}{\sum_{j>d_k}\lambda_{k,j}^2},
\]
the transfer-ratio lower bound becomes
\begin{equation}\label{eq:loopfm-tr-lb-general}
\mathrm{TR}_{\text{LoopFM}}
\;\ge\;
\frac{ - \tau_2\,\mathcal{I}_{\mathrm{temporal}} + (1 - \eta_1)\, \underline{\kappa}_{\mathrm{gap}}^{\mathrm{hist}}\,\delta}
{\bar{\kappa}_{\mathrm{gap}}\,\delta + \bar C_{\mathrm{over}}\,\sigma^2\,\xi_1 - \underline C_{\mathrm{over}}\,\sigma^2\,\xi_2}.
\end{equation}
Setting $d_k = m_k$ (input feature dimension) and $R_{m_k} = p_k - m_k$ (bi-level covariance) recovers Theorem~\ref{thm:loopfm-tr}.
\end{theorem}

\begin{proof}
The proof of Theorem~\ref{thm:loopfm-tr} carries through verbatim with $\xi_k = d_k/n + n/R_{d_k}(\Sigma_k)$ in place of $m_k/n + n/(p_k-m_k)$.
The only change is that the two-sided A1 envelope uses the general noise functional $\Psi_k = \xi_k$ from Theorem~\ref{thm:linearized-overparam}, without specializing the split index or effective rank.
\end{proof}

\section{Impact of FM Capacity}
\label{sec:rq6}

Table~\ref{tab:fm_capacity} scales DMIN across five capacity levels (1.9M--31.5M parameters).
KD+LoopFM performance is \emph{stable} on TaobaoAd: AUC ranges from 0.6243 to 0.6284 (spread 0.0041), meaning even a small FM produces embeddings rich enough for strong LoopFM gains.
The $\Delta$AUC (KD+LoopFM minus KD) remains large across all capacity levels (0.0332--0.0380), showing no systematic trend with FM size---LoopFM's embeddings provide strong complementary value regardless of FM capacity.
Standalone DMIN scaling shows AUC saturates at 31--65M params and degrades beyond due to overfitting, motivating knowledge transfer over brute-force scaling.
We acknowledge that this study might not be representative of real industrial FM with trillions of parameters and billions or even larger number of examples.

\begin{table}[h]
\caption{LoopFM gain across FM capacity levels (TaobaoAd, FM: DMIN variants, VM: Deep\underline{FM}). FM capacity is scaled by varying embedding dimensions and MLP widths/depths while keeping the architecture fixed. $\Delta$AUC = AUC(KD+LoopFM) $-$ AUC(KD).}
\label{tab:fm_capacity}
\centering
\small
\resizebox{\textwidth}{!}{%
\begin{tabular}{ccccccc}
\toprule
Params (M) & Emb.\ dim & DNN hidden & AUC (KD) & AUC (LoopFM) & AUC (KD+LoopFM) & $\Delta$AUC \\
\midrule
1.9 & 2 & [32, 16, 8] & 0.5888 & 0.6159 & 0.6243 & +0.0356 \\
3.8 & 4 & [64, 32, 16] & 0.5900 & 0.6179 & 0.6266 & +0.0366 \\
7.7 & 8 & [128, 64, 32] & 0.5952 & 0.6142 & \textbf{0.6284} & +0.0332 \\
15.5 & 16 & [256, 128, 64] & 0.5887 & 0.6172 & 0.6267 & +0.0380 \\
31.5 & 32 & [512, 256, 128] & 0.5901 & 0.6175 & 0.6272 & +0.0371 \\
\bottomrule
\end{tabular}%
}
\end{table}

\section{Interaction Embeddings vs.\ Item-Only Embeddings}
\label{sec:rq7}

A natural question is whether LoopFM's gains come from the FM's learned \emph{user-item interaction} representations, or whether simpler item-side-only embeddings---raw feature lookup table outputs without DNN processing---would suffice.
We ablate this by extracting only the 6 pure item-side feature embeddings from the FM's embedding layer (adgroup\_id, cate\_id, campaign\_id, customer, brand, price; $6 \times 32 = 192$ dimensions), compressing them to $d{=}32$ via the same autoencoder architecture used for dnn\_hidden\_0, and building LoopFM sequences identically to the standard pipeline.

\begin{table}[h]
\caption{Interaction embeddings (dnn\_hidden\_0) vs.\ item-only embeddings on TaobaoAd (FM: DMIN, VM: Deep\underline{FM}).
Item-side embeddings use 6 pure ad feature lookups (no user/context features, no DNN processing), compressed via the same AE architecture as dnn\_hidden\_0.}
\label{tab:adside_ablation}
\centering
\small
\begin{tabular}{lcc}
\toprule
Method & AUC & LogLoss \\
\midrule
Baseline (no distill) & 0.5882 & 0.1977 \\
KD & 0.5980 & 0.1964 \\
\midrule
LoopFM-only (item-side) & 0.6172 & 0.1956 \\
LoopFM-only (dnn\_hidden\_0) & 0.6245 & 0.1955 \\
\midrule
KD + LoopFM (item-side) & 0.6291 & 0.1942 \\
KD + LoopFM (dnn\_hidden\_0) & \textbf{0.6342} & \textbf{0.1938} \\
\bottomrule
\end{tabular}
\end{table}

Table~\ref{tab:adside_ablation} shows that item-only embeddings capture most of LoopFM's gain: KD + LoopFM (item-side) achieves AUC 0.6291, covering 86\% of the full interaction embedding's improvement over the KD baseline ($+$0.0311 vs.\ $+$0.0362).

\paragraph{Why item-side embeddings work well on TaobaoAd.}
On TaobaoAd, the FM (DMIN) is a relatively shallow model with 3 DNN layers, and the FM and VM share identical features.
In this setting, the FM's DNN layers add limited cross-feature interaction information beyond what the raw item embeddings already encode---the item identity features (adgroup\_id, campaign\_id) are already highly predictive.

\paragraph{Why the gap should widen in practice.}
This result is a single data point on a public benchmark with a shallow FM.
In industrial FMs with much deeper and wider sequence layers, multi-task heads, and cross-domain features, the DNN (or more general interaction) layers contribute a much larger fraction of the representation's total information---fusing user context, behavioral signals, and feature interactions that raw item lookups cannot capture.
A 3-layer DMIN barely transforms raw embeddings, so item-only is close to the full representation; a trillion-parameter FM produces interaction representations that are qualitatively different from feature lookups, and the gap should grow accordingly.
Identifying the sweet spot---deep enough to capture rich interactions, but not so deep that compression discards too much information (cf.\ the layer selection ablation in Table~\ref{tab:ablation})---remains an open question.

\section{Embedding Analysis}
\label{app:embedding_analysis}

To understand what information LoopFM embeddings encode, we conduct three analyses on TaobaoAd (FM: DMIN, autoencoder bottleneck $d{=}32$).

\paragraph{Embedding structure.}
We analyze 50,000 compressed embeddings from Day~5 (FM trained on Days~1--4).
The embeddings exhibit full effective rank: all 32 dimensions contribute meaningfully, with the top-4 and top-8 singular values explaining 28.0\% and 47.3\% of total variance respectively.
This confirms that the autoencoder distributes information across all dimensions rather than concentrating it in a low-rank subspace.
The $\tanh$ activation bounds values to $[-0.92, 1.00]$ with mean L2 norm $1.71 \pm 0.14$.

\paragraph{t-SNE visualization.}
Figure~\ref{fig:tsne} visualizes 20,000 compressed embeddings via t-SNE, colored by FM soft-label (left) and ground-truth click label (right).
The soft-label coloring reveals a smooth gradient from low predicted CTR (blue) to high (red), confirming that the 32-dimensional compressed space preserves the FM's calibration structure.
The ground-truth coloring shows partial separation between clicked (red) and non-clicked (blue) samples, consistent with the inherent noise in click prediction.
The final t-SNE KL divergence of 4.31 indicates a good fit between the high-dimensional and low-dimensional affinity structures.

\begin{figure}[h]
\centering
\includegraphics[width=\textwidth]{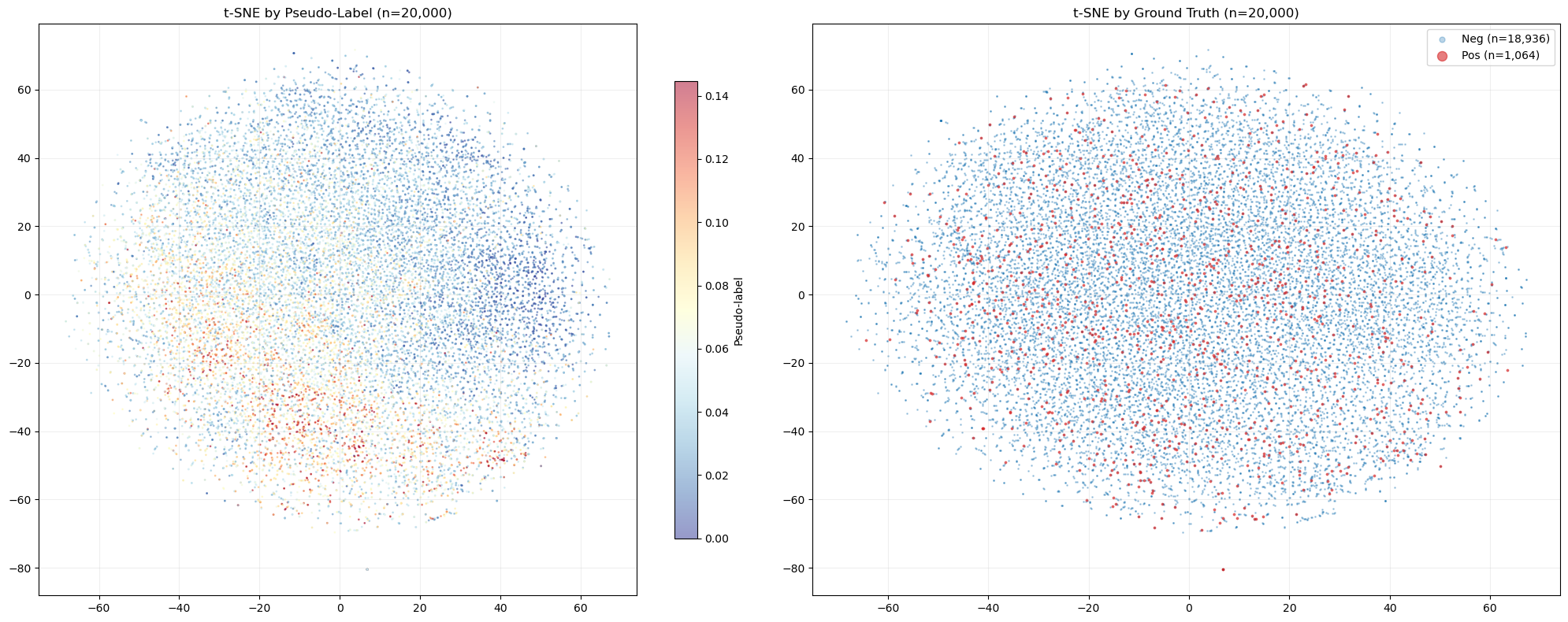}
\caption{t-SNE visualization of compressed LoopFM embeddings ($d{=}32$, 20K samples from TaobaoAd Day~5). \textbf{Left:} Colored by FM soft-label (predicted CTR), showing smooth calibration gradients. \textbf{Right:} Colored by ground-truth click label, showing partial separation.}
\label{fig:tsne}
\end{figure}

\paragraph{Probing experiment.}
We use the FM's soft-label predictions (scalar soft labels) as a proxy for the FM's learned knowledge.
On Day~5 test data, soft-labels achieve AUC 0.6136 as a predictor of ground-truth click labels, confirming they carry substantial task-relevant signal.
Figure~\ref{fig:dim_corr} shows per-dimension Pearson correlations with soft-labels (left) and ground-truth labels (right).
The correlation patterns are strongly aligned ($\rho = 0.95$), indicating that the compressed embeddings preserve the FM's ranking of feature importance almost perfectly.
This validates that the autoencoder retains task-relevant information despite aggressive dimensionality reduction.

\begin{figure}[h]
\centering
\includegraphics[width=\textwidth]{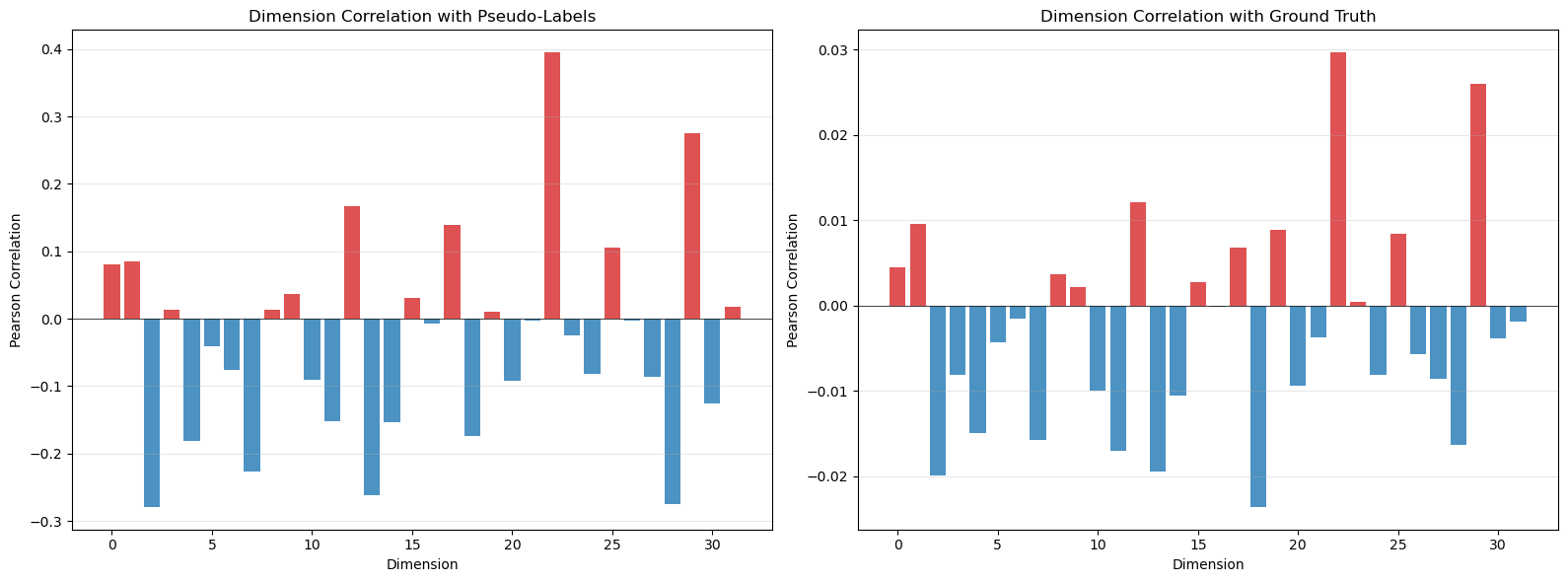}
\caption{Per-dimension Pearson correlation of compressed embeddings with FM soft-labels (left) and ground-truth click labels (right). The correlation patterns are strongly aligned ($\rho{=}0.95$), confirming the autoencoder preserves the FM's learned feature importance structure.}
\label{fig:dim_corr}
\end{figure}

\paragraph{Attention weight analysis.}
We analyze the multi-head attention weights (4 heads) in the VM's DMIN-style sequence encoder over LoopFM positions on 163,840 test samples (mean sequence length 16.7).

Figure~\ref{fig:attn_temporal} shows that attention weight decreases monotonically with temporal distance to the target interaction: entries within 10 minutes receive ${\sim}3\times$ higher attention than those 4--8 hours away, confirming strong recency bias consistent with user interest decay.
All four heads exhibit nearly identical temporal profiles, suggesting temporal recency is a universal signal rather than a head-specialized one.

\begin{figure}[h]
\centering
\includegraphics[width=0.75\textwidth]{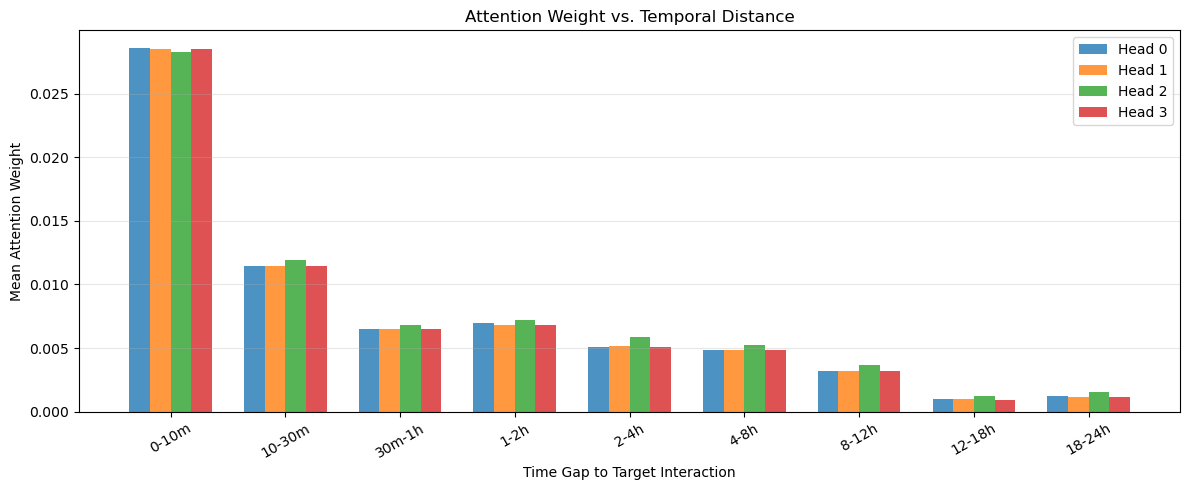}
\caption{Attention weight vs.\ temporal distance to target interaction. All 4 heads show strong recency bias: recent entries ($<$10 min) receive ${\sim}3\times$ higher attention than distant ones ($>$4h).}
\label{fig:attn_temporal}
\end{figure}

Figure~\ref{fig:attn_semantic} examines whether attention correlates with semantic relevance by comparing weights assigned to same-category vs.\ different-category historical interactions.
All four heads show a consistent 1.13--1.16$\times$ attention lift for same-category items, confirming that the sequence encoder learns to attend preferentially to semantically relevant history entries.
Brand matching shows a weaker but consistent 1.04--1.06$\times$ lift, reflecting that category is a stronger relevance signal than brand in TaobaoAd.
Interestingly, attention weights for historically clicked items are slightly \emph{lower} than for unclicked items (0.84--0.97$\times$ ratio), suggesting the model leverages both positive and negative interaction signals.

\begin{figure}[h]
\centering
\includegraphics[width=\textwidth]{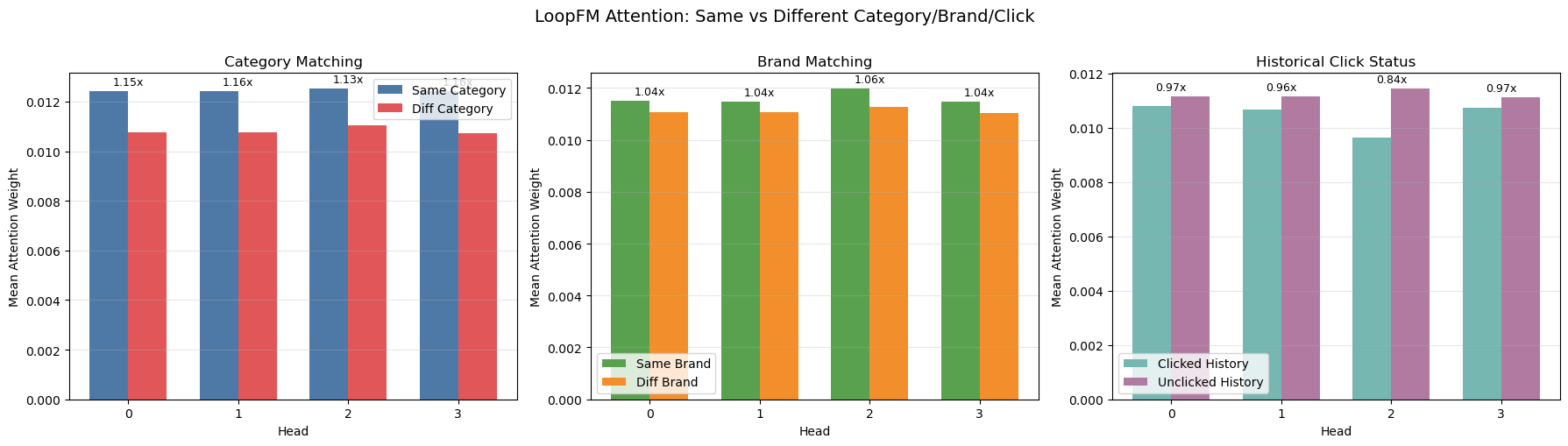}
\caption{Attention weight comparison by semantic matching. \textbf{Left:} Same vs.\ different category (1.13--1.16$\times$ lift). \textbf{Center:} Same vs.\ different brand (1.04--1.06$\times$ lift). \textbf{Right:} Historically clicked vs.\ unclicked items (0.84--0.97$\times$ ratio), showing the model uses negative signals.}
\label{fig:attn_semantic}
\end{figure}

\section{Future Directions}
\label{app:future}

The modularity of LoopFM's three-stage design naturally suggests extensions along each axis.
We highlight several promising directions.

\paragraph{Alternative structuring: item sequences and graph-based features.}
Our experiments use user-level temporal sequences for Stage~3. An equally natural choice is \emph{ad-side sequences}: grouping FM embeddings by item (ad) rather than user to capture how different users have interacted with the same ad over time, providing item-level demand signals.
Beyond flat sequences, FM embeddings can be organized into \emph{graph-structured neighborhoods}---for example, constructing user-item bipartite graphs where edges carry FM embeddings, enabling graph neural network encoders in the VM.

\paragraph{Cross-domain transfer.}
LoopFM's decoupled architecture naturally supports cross-domain knowledge transfer: an FM trained on one domain (e.g., organic content) can generate embeddings consumed by VMs in a different domain (e.g., ads).
Because the VM consumes FM embeddings as opaque input features, no alignment of prediction objectives or feature schemas is needed

\paragraph{Self-LoopFM.}
An intriguing direction is \emph{self-LoopFM}: having the FM itself consume its own historical embeddings as input features in subsequent training iterations, creating a self-improving loop.

\paragraph{Checkpoint-aligned embeddings.}
Our FM checkpoint frequency ablation (Section~\ref{sec:rq5}) reveals that embedding consistency within a user's sequence matters more than individual embedding freshness.
When the FM is updated incrementally, its embeddings drift (paired cosine similarity decays to ${\sim}$0.80 over 3 days), degrading sequence encoder performance.
Techniques such as anchored training, projection alignment (learning a linear map from new-checkpoint space to old-checkpoint space), or EMA-based checkpoint interpolation could enable LoopFM to benefit from fresher FM knowledge without sacrificing within-sequence consistency.

\paragraph{Vector quantization.} Currently we use scalar quantization such as INT4 for storage efficiency. While K-means can minimize INT4 quantization loss, it is intriguing to pursue quantization in vector space, with notable approaches such as RQ-VAE~\citep{lee2022rqvae}.

\end{document}